\documentclass[10pt,twocolumn,letterpaper]{article}

\usepackage{cvpr}      
\definecolor{cvprblue}{rgb}{0.21,0.49,0.74}
\usepackage[pagebackref,breaklinks,colorlinks,allcolors=cvprblue]{hyperref}
\usepackage{booktabs}   
\usepackage{multirow}   
\usepackage{colortbl}   
\usepackage[table]{xcolor} 
\usepackage{graphicx}   
\usepackage{adjustbox}
\usepackage{amsmath,amssymb,mathtools} 
\usepackage{algorithm}
\usepackage[noend]{algpseudocode} 

\usepackage{amsthm}  

\newtheorem{theorem}{Theorem}[section]
\newtheorem{lemma}[theorem]{Lemma}
\newtheorem{proposition}[theorem]{Proposition}

\newtheorem{definition}[theorem]{Definition}
\newtheorem{assumption}[theorem]{Assumption}

\theoremstyle{plain}
\theoremstyle{remark}
\newtheorem{remark}[theorem]{Remark}
\DeclareMathOperator*{\softmax}{softmax}

\DeclareMathOperator{\quantile}{Quantile}

\def\paperID{} 
\def\confName{CVPR}
\def\confYear{2026}

\title{FlashVLM: Text-Guided Visual Token Selection for Large Multimodal Models}

\author{
Kaitong Cai$^{1,*}$,
Jusheng Zhang$^{1,*}$,
Jing Yang$^{1}$,
Yijia Fan$^{1}$,
Pengtao Xie$^{2}$,
Jian Wang$^{3}$,
Keze Wang$^{1}$\\
$^{1}$Sun Yat-sen University \quad
$^{2}$University of California, San Diego \quad
$^{3}$Snap Inc.\\
$^{*}$Equal contribution.
}
\begin{document}
\maketitle
\begin{abstract}
Large Vision-Language Models (VLMs) typically process hundreds or even thousands of visual tokens per image or video frame, incurring quadratic attention cost and substantial information redundancy. Existing token-reduction methods either ignore the textual query or depend on deep attention maps, whose instability under aggressive pruning often leads to degraded semantic alignment. We introduce \textbf{FlashVLM}, a text-guided visual token selection framework that dynamically adapts visual inputs to the query. Instead of relying on noisy attention weights, FlashVLM computes an explicit cross-modal similarity between projected image tokens and normalized text embeddings in the LLM space, and fuses this extrinsic relevance with intrinsic visual saliency through log-domain weighting and temperature-controlled sharpening. A diversity-preserving partition further retains a minimal yet representative set of background tokens to maintain global context. Under identical token budgets and evaluation protocols, FlashVLM achieves \emph{beyond-lossless} compression, i.e., slightly surpassing the unpruned baseline while pruning up to 77.8\% of visual tokens on LLaVA-1.5, and maintaining 92.8\% accuracy even under 94.4\% compression. Extensive experiments on 14 image–video benchmarks demonstrate that FlashVLM delivers state-of-the-art efficiency–performance trade-offs while maintaining strong robustness and broad generalization across mainstream VLMs.
\end{abstract}
    
\section{Introduction}
\label{sec:introduction}

Large Vision Language Models (VLMs), e.g., LLaVA~\cite{LLaVA-1.51,LLaVA-PruMerge,llava,z1,z2,z9} and GPT-4V~\cite{GPT-4o}, have achieved remarkable success by coupling powerful vision encoders with Large Language Models (LLMs)~\cite{zs1,zs2,zs3,zs4,LLaVA-PruMerge,zs5,zhuyil,zhuyil2,z11,z4,z14}. 
The prevailing paradigm converts each image or video frame into a long sequence of visual tokens—often hundreds (e.g., $24\times24=576$) or even thousands (e.g., over 2K tokens for multi-frame videos)—and then concatenates them with text tokens. 
However, this ``feed-all-tokens'' design is fundamentally \emph{token-inefficient}: the self-attention in LLMs scales quadratically ($O(L^2)$) with the total sequence length, so every additional visual token inflates both computational and memory costs~\cite{zhuyil,zhuyil2,zhuyil3,Papa_2024,zs5,Z5,z6,z8}. 
This inefficiency is particularly acute for high-resolution inputs and video understanding, which are known to contain substantial spatial and temporal redundancy~\cite{bertasius2021spacetimeattentionneedvideo,ryoo2022tokenlearner8learnedtokens}.

\begin{figure}[t]
    \centering
    \includegraphics[width=\columnwidth]{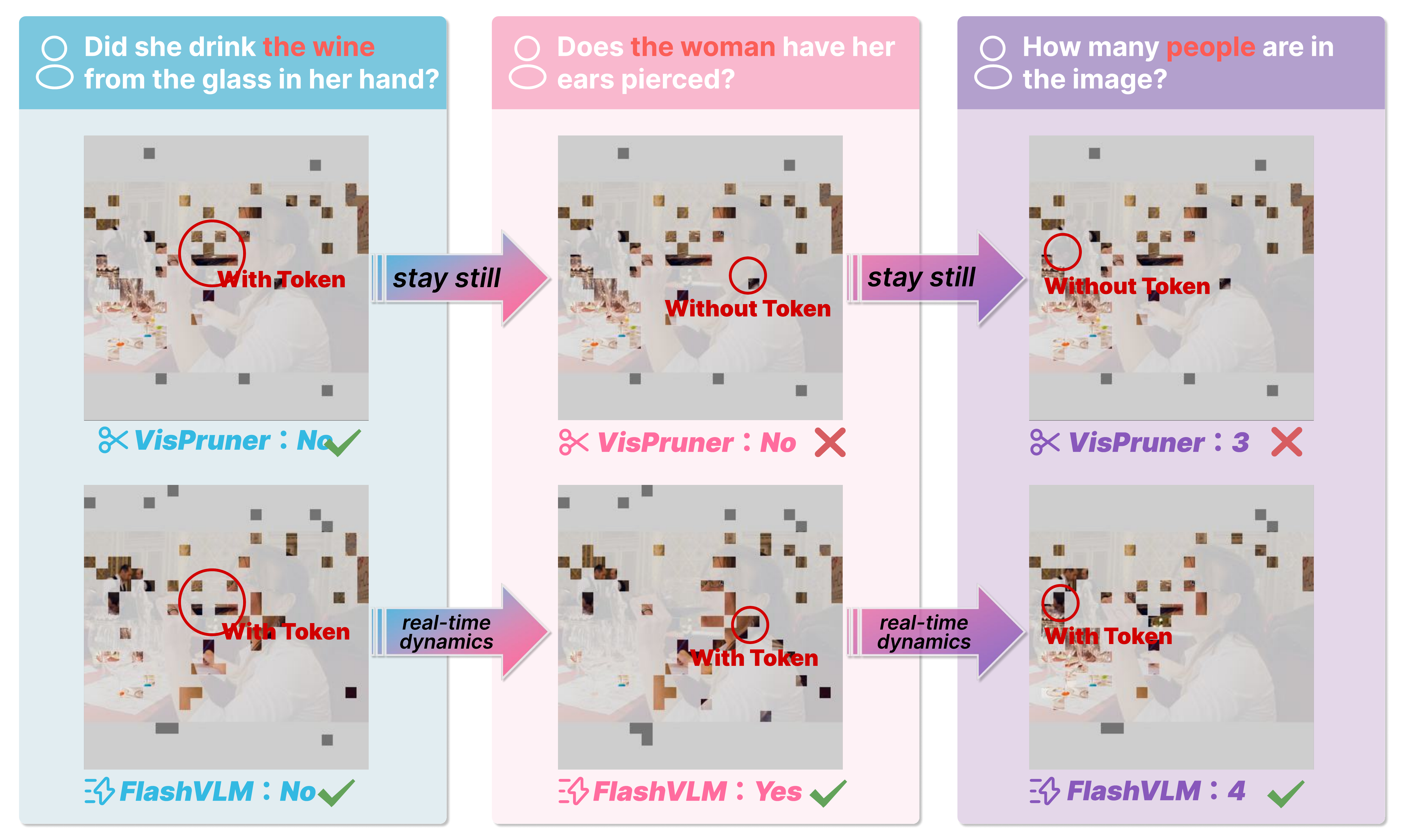} 
    \vspace{-20pt}
    \caption{VisPruner keeps its attention static and insensitive to question semantics, while FlashVLM dynamically selects key tokens based on textual cues, achieving more accurate region focus and stronger semantic alignment.}
    \label{fig:contrast} 
    \vspace{-15pt}
\end{figure}

Beyond raw computation, this holistic representation also creates a \textbf{semantic redundancy} problem. 
In most real-world scenarios, the user’s query is highly localized, focusing on a small subset of objects or regions (e.g., ``What is the dog on the left doing?''). 
Yet existing VLMs must still propagate and attend over \emph{all} visual tokens, including hundreds of \textbf{query-irrelevant} patches from backgrounds, skies, or distractor objects. 
These tokens not only waste quadratic attention budget, but also dilute attention and introduce noise into the reasoning process. 
Empirically, we observe that judiciously filtering such noisy tokens can lead to a counter-intuitive phenomenon: \emph{``beyond-lossless''} compression, where performance can slightly exceed that of the unpruned baseline under identical evaluation protocols~\cite{frankle2019lotterytickethypothesisfinding,graesser2022statesparsetrainingdeep,hooker2021compresseddeepneuralnetworks,cheng2020surveymodelcompressionacceleration,z7}, because irrelevant distractors are suppressed.
\textbf{Problem and desiderata.}
We therefore study the following problem: given a sequence of visual tokens $\mathbf{v} = \{v_1, \ldots, v_N\}$ from a frozen vision encoder and a text query $x$, how can we select a compact subset $\mathbf{v}_S$ to feed into the LLM such that (1) task performance is preserved or even improved; (2) the selection procedure is \emph{efficient} and adds negligible overhead compared to the ensuing savings; and (3) the mechanism is \emph{architecture-agnostic}, requiring no modifications to transformer layers or attention operators?
In particular, we focus on a \emph{single-shot} selection scheme at the encoder–decoder interface, a far more deployable approach than per-layer sparsification.
\begin{quote}
(Q): \emph{Why do existing visual token reduction techniques still fail to offer a selector that is query-aware, stable under heavy pruning, efficient, and deployable without modifying VLM backbones?}
\end{quote}

A variety of visual token reduction methods have been explored, which can be grouped into three main families:

\textbf{(1) Text-agnostic pruning.}
Methods such as ToMe~\cite{bolya2023tokenmergingvitfaster}, VisionZip~\cite{VisionZip}, and VisPruner~\cite{VisPruner} rely solely on \emph{intrinsic} visual cues (e.g., saliency, similarity, or [CLS]–patch attention) to merge or drop tokens, independent of the text query. 
While efficient, they are inherently \emph{query-agnostic}: subtle yet query-relevant regions may be removed, causing \textbf{brittle performance} under heavy compression.

\textbf{(2) Text-guided pruning via attention.}
Approaches like FastV~\cite{FastV} and SparseVLM~\cite{SparseVLM} use cross-/self-attention maps to derive relevance, but such signals are often \textbf{unstable} under aggressive pruning due to head sparsity, positional bias, and query-length sensitivity. 
Many also operate \emph{inside} the LLM, performing layer-wise sparsification that increases overhead and \textbf{breaks compatibility} with optimized kernels such as FlashAttention.

\textbf{(3) Structural merging and layer-wise sparsification.}
A third line compresses tokens within each transformer block via merging, clustering, or rank-based pruning, but requires architectural modifications, complicating deployment and limiting portability across VLM backbones.

Designing a token selector that is simultaneously \emph{query-aware}, \emph{stable}, \emph{efficient}, and \emph{non-intrusive} thus remains an open challenge.
In this paper, we introduce \textbf{FlashVLM}, a \textbf{text-guided visual token selection} framework that dynamically and efficiently tailors the visual input to the LLM based on the query. 
Unlike prior relevance estimators that \textbf{rely} on deep attention maps, FlashVLM computes an \emph{explicit cross-modal similarity} between projected image tokens and normalized text embeddings in the LLM space, without accessing internal attention weights or gradients. 
This \emph{attention-light} design yields a more stable and interpretable query signal, especially under aggressive compression. 
Concretely, FlashVLM first computes an ``extrinsic'' query–token relevance score via embedding-space similarity, and then fuses it with ``intrinsic'' saliency derived from the vision encoder via log-domain weighting with temperature-controlled sharpening and top-$p$ gating. 
A diversity-preserving selection step further \textbf{mitigates} local redundancy by retaining a small, non-redundant subset of background tokens that preserves essential global context, rather than collapsing \textbf{onto only} a few highly salient patches. 
Crucially, the entire selection procedure is performed \emph{once} at the encoder–decoder boundary, requires \emph{no modifications} to transformer blocks, and remains \textbf{fully compatible} with FlashAttention-based acceleration.

\begin{figure*}[t]
    \centering
    \includegraphics[width=0.85\textwidth]{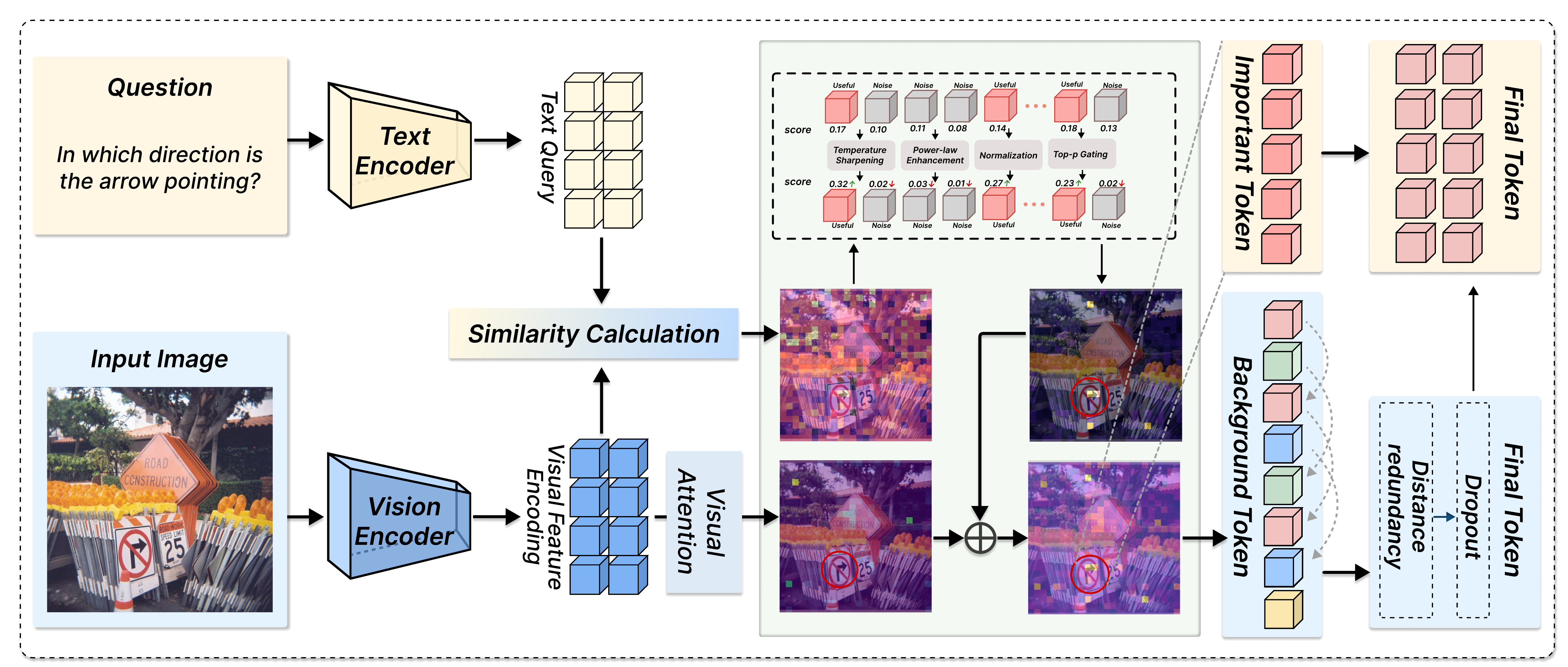} 
    \vspace{-10pt}
    \caption{The model encodes text and image features, then combines cross-modal similarity with visual attention to partition visual tokens into important, background, and diverse tokens. A semantic reweighting and redundancy elimination module selects the final key tokens from these groups, enabling efficient and semantically aligned visual reasoning.}
    \label{fig:token_flow} 
    \vspace{-10pt}
\end{figure*}

We evaluate FlashVLM on 14 multimodal datasets spanning image (VQAv2, GQA, MME, MMBench/CN) and video QA (TGIF-QA, MSVD-QA, MSRVTT-QA). 
Under a \textbf{unified evaluation protocol} with identical token budgets and resolutions, FlashVLM delivers strong efficiency--performance trade-offs: on LLaVA-1.5 it achieves \textbf{100.60\%} relative accuracy while pruning \textbf{77.8\%} of visual tokens (128 kept), slightly outperforming the unpruned model by suppressing query-irrelevant noise. 
It also remains \textbf{robust} under \textbf{94.4\%} compression (32 tokens), with consistent gains across LLaVA, Qwen-VL, InternVL, and CogVLM for both image and video tasks.

\noindent\textbf{Our main contributions are threefold:}
\begin{itemize}
    \item \textbf{FlashVLM.} 
    We propose a lightweight text-guided token selector that fuses intrinsic saliency with embedding-based query relevance, offering a stable, attention-independent signal with minimal overhead.

    \item \textbf{State-of-the-art Compression.} 
    FlashVLM achieves ``beyond-lossless'' performance at moderate budgets and remains robust under extreme pruning, all under strictly identical evaluation protocols.

    \item \textbf{Broad generality.}
    Our method consistently improves diverse VLM backbones across image and video tasks \emph{without modifying transformer layers} or requiring task-specific tuning.
\end{itemize}

\section{Related Work}
\label{sec:formatting}

\paragraph{Visual Token Reduction and Efficiency in VLMs.}
Recent Vision-Language Models (VLMs)~\cite{zs1,zs2,zs3,zs5,zhang2025cfvlmcounterfactualvisionlanguagefinetuning,9880220,z10,z15} encode high-resolution images or multi-frame videos into dense patch embeddings, resulting in thousands of visual tokens and quadratic attention cost. To address this, a series of token compression and merging strategies has been explored. Early works such as ToMe~\cite{bolya2023tokenmergingvitfaster} and VisionZip~\cite{VisionZip} compress tokens through intra-image redundancy analysis, while VisPruner~\cite{VisPruner} introduces saliency-guided pruning for adaptive visual sparsity. However, these approaches are largely text-agnostic and fail to consider the semantic relevance of tokens to the input query, leading to the removal of semantically critical regions. Other compression efforts~\cite{liang2023clusterformerclusteringuniversalvisual, Papa_2024,z12,z13,z21} focus on attention-based clustering, but their performance deteriorates under aggressive pruning due to instability in attention weights and limited cross-modal guidance.

\paragraph{Text-Guided Token Selection and Cross-Modal Relevance.}
Recent advances have incorporated textual cues into token selection to improve query awareness. FastV~\cite{FastV} and SparseVLM~\cite{SparseVLM} adopt attention-based cross-modal relevance estimation to retain query-related regions. Despite improved alignment, attention-derived signals are inherently unstable—affected by head sparsity, query-length sensitivity, and noise accumulation in deep layers. Other approaches leverage gradient-based saliency~\cite{ScienceQA-IMG,sanh2020movementpruningadaptivesparsity,z18,z19} or CLIP-based similarity~\cite{cLIP,du2022surveyvisionlanguagepretrainedmodels,z16,z17} to integrate linguistic guidance. However, these methods require additional backpropagation or suffer from semantic drift when vision–language features are not co-embedded. In contrast, FlashVLM introduces an explicit and attention-light similarity formulation between projected visual tokens and normalized text embeddings in the LLM space, enabling stable, interpretable, and low-cost relevance estimation even under heavy compression.

\paragraph{Balancing Efficiency and Semantic Fidelity.}
Beyond pruning, a concurrent research line explores maintaining global context and semantic coherence after compression. PruMerge+~\cite{LLaVA-PruMerge} and VisionZip~\cite{VisionZip} attempt to preserve holistic representations via clustering or token regrouping, while VisPruner~\cite{VisPruner} employs lightweight reweighting to stabilize saliency maps. However, these approaches often exhibit \emph{local collapse}, i.e., focusing too narrowly on high-response regions while losing background cues critical for reasoning. FlashVLM mitigates this issue through diversity-preserving partitioning, ensuring that both important and representative background tokens are maintained. 

\section{Methodology}
\label{sec:method}

Our \textbf{FlashVLM} introduces a dynamic, text-guided mechanism for selecting visual tokens during inference. 
Given a frozen vision encoder and an LLM, FlashVLM replaces the full visual token sequence with a compact, query-aware subset~\cite{jin2024efficientmultimodallargelanguage,du2022surveyvisionlanguagepretrainedmodels,li2023multimodalfoundationmodelsspecialists}. 
This strategy aims to preserve task performance while significantly reducing computational overhead. 
By discarding query-irrelevant tokens before they enter the LLM, FlashVLM reduces the quadratic attention cost and suppresses visual noise. 
As shown in our experiments (Sec.~\ref{sec:experiments}), the method remains empirically robust even under high compression ratios.
As illustrated in Fig.~\ref{fig:token_sample}, our method comprises two principal stages:
\begin{itemize}
    \item \textbf{Query-Guided Relevance Fusion (Sec.~\ref{subsec:fusion}).} We derive a fused relevance score $S_{\text{fused}}$ for each visual patch by combining \emph{intrinsic} visual saliency (query-agnostic) with \emph{extrinsic} query relevance (query-conditioned) through log-domain fusion. While our formulation is compatible with various similarity measures \cite{oord2019representationlearningcontrastivepredictive,hjelm2019learningdeeprepresentationsmutual}, we employ dot-product similarity by default for its efficiency and stability.
    \item \textbf{Diversity-Preserving Partitioning (Sec.~\ref{subsec:partitioning}).} Given $S_{\text{fused}}$, we partition tokens into an ``important'' set and a ``residual'' set. The important set contains top-scoring tokens, while an iterative pruning procedure over the residual set retains a small, non-redundant subset of background tokens to maintain global context.
\end{itemize}
Crucially, FlashVLM operates in an \emph{attention-light} manner: all relevance signals are computed directly from feature embeddings without modifying or repeatedly sparsifying internal attention maps.

\begin{figure*}[!t]
    \centering
    \includegraphics[width=\textwidth]{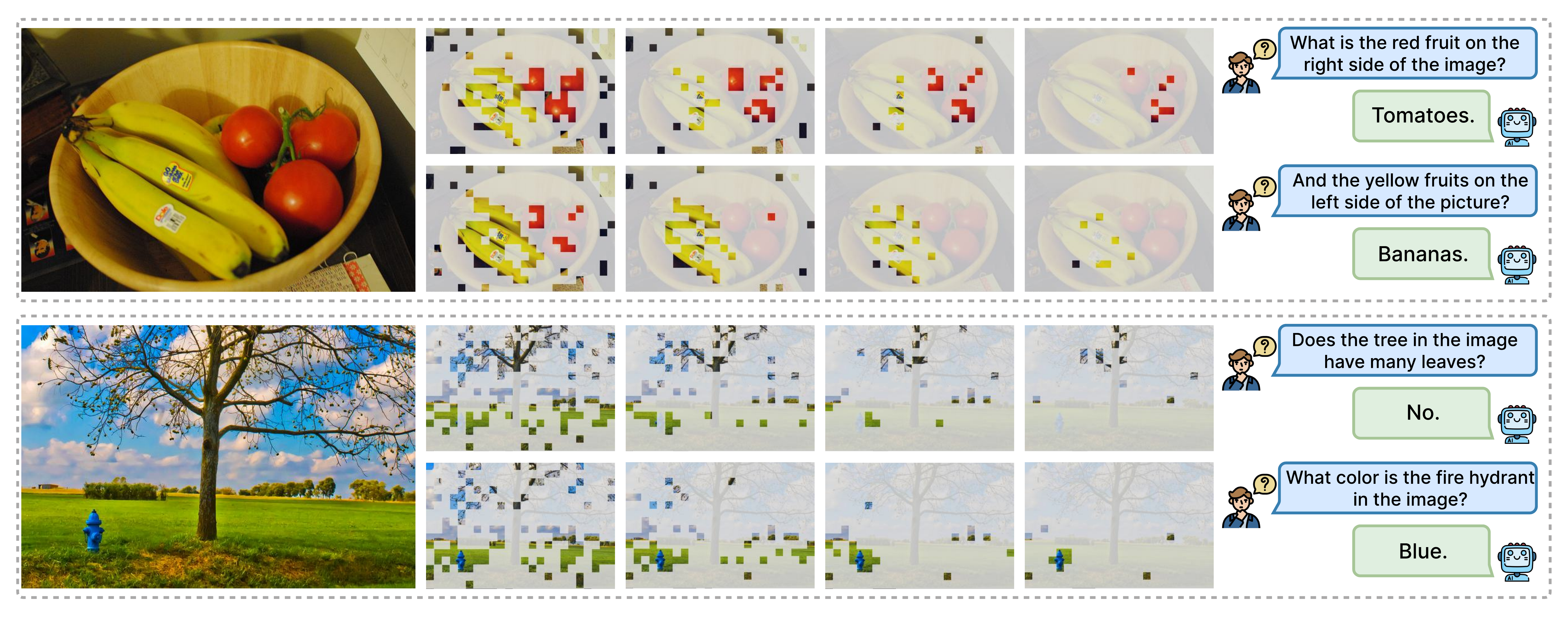}
    \vspace{-25pt}
    \caption{The model’s focus shifts according to each question: in the fruit image, it attends to tomatoes when asked about “the red fruit” and switches to bananas for “the yellow fruits”; in the outdoor scene, it focuses on the tree crown for leaf-related queries and moves to the blue fire hydrant when asked about its color, demonstrating strong question-dependent dynamic adaptation.}
    \label{fig:token_sample}
    \vspace{-10pt}
\end{figure*}

\subsection{Query-Guided Relevance Fusion}
\label{subsec:fusion}

Let $N$ visual patch features $\mathbf{V} = \{\mathbf{v}_i\}_{i=1}^N \in \mathbb{R}^{N \times D_v}$ be extracted from the vision tower, and $L_t$ text token embeddings $\mathbf{T}_{\text{raw}} = \{\mathbf{t}_j\}_{j=1}^{L_t} \in \mathbb{R}^{L_t \times D_{\text{llm}}}$ be generated by the LLM's embedder. $D_v$ and $D_{\text{llm}}$ represent the visual and LLM feature dimensions, respectively.

\textbf{Intrinsic Visual Saliency ($S_{\text{intrinsic}}$)} quantifies the query-agnostic importance of each patch. 
We derive this signal, $A$, from the vision encoder's attention mechanism (e.g., from the final layer) by averaging attention weights across all heads. 
To normalize the scores, we apply min–max normalization $N_{\text{minmax}}$:
\begin{equation}
  S_{\text{intrinsic}} = N_{\text{minmax}}(A)
  = \frac{A - \min(A)}{\max(A) - \min(A) + \epsilon},
  \label{eq:intrinsic}
\end{equation}
where $\epsilon$ is a small constant (e.g., $10^{-6}$) for numerical stability. 
This score reflects inherent visual prominence independent of the current query.

\textbf{Extrinsic Query Relevance ($S_{\text{extrinsic}}$)} measures how well each patch aligns with the text query in the shared LLM space, providing a complementary cross-modal signal.

\paragraph{Projection and Query Gating.}
We first project visual features into the LLM space using the multimodal projector $W_{\text{proj}}$ and apply $\ell_2$-normalization:
\begin{equation}
    \mathbf{V}_{\text{proj}} = N_{L_2}\bigl(W_{\text{proj}}(\mathbf{V})\bigr) \in \mathbb{R}^{N \times D_{\text{llm}}}.
    \label{eq:proj}
\end{equation}
For the raw text embeddings $\mathbf{T}_{\text{raw}}$, we employ norm-based gating to emphasize semantically informative tokens (e.g., nouns) over low-magnitude ones (e.g., stopwords). 
For each token $\mathbf{t}_j$, the gate $g_j$ is defined as
\begin{equation}
    g_j = \frac{\|\mathbf{t}_j\|_{2,\text{raw}}}{\max_k \|\mathbf{t}_k\|_{2,\text{raw}} + \epsilon}.
    \label{eq:gate}
\end{equation}
The gated embeddings $\mathbf{T}_{\text{gated}} = \{g_j \cdot \mathbf{t}_j\}_{j=1}^{L_t}$ are subsequently $\ell_2$-normalized to yield $\mathbf{T}_{\text{norm}} = N_{L_2}(\mathbf{T}_{\text{gated}})$. 
In the absence of a text query, $S_{\text{extrinsic}}$ is set to zero, and the method gracefully degrades to attention-only pruning.

\paragraph{Similarity Aggregation.}
We compute a cross-modal similarity matrix $S_{\text{cross}} \in \mathbb{R}^{N \times L_t}$ via dot product:
\begin{equation}
    S_{\text{cross}} = \mathbf{V}_{\text{proj}} \cdot \mathbf{T}_{\text{norm}}^{\top}.
\end{equation}
For each visual patch $\mathbf{v}_i$, we then aggregate its similarity to all text tokens into a single score, $s_{\text{text}}[i]$, using a temperature-controlled weighted sum:
\begin{equation}
\resizebox{0.95\linewidth}{!}{$%
  w_{ij} = \softmax\nolimits_j\!\bigl(S_{\text{cross}}[i,j]/\tau_{\text{agg}}\bigr),\quad
  s_{\text{text}}[i] = \sum_{j=1}^{L_t} w_{ij}\,S_{\text{cross}}[i,j],%
$}
\label{eq:agg}
\end{equation}
where $\tau_{\text{agg}} = 0.05$ is chosen empirically to sharpen intra-query attention (see ablations in Sec.~\ref{sec:ablation}).

\paragraph{Contrastive Sharpening.}
To better separate query-relevant patches from low-signal ones, we apply a lightweight multi-stage contrastive sharpening process to the aggregated similarity scores $s_{\text{text}}$.
This process consists of three steps:
\textbf{Temperature Sharpening:} A low-temperature softmax increases contrast among patch-level responses:
    $ s_1 = \mathrm{softmax}\!\left(\frac{s_{\text{text}}}{\tau_{\text{sharp}}}\right), \quad \tau_{\text{sharp}} = 0.01. $
\textbf{Power-Law Enhancement:} Strong responses are further highlighted using a power function:
    $ s_2 = s_1^\gamma, \quad \gamma = 2.5. $
\textbf{Normalization:} Values are rescaled to a consistent [0, 1] range:
    $ s_3 = N_{\text{minmax}}(s_2). $

Finally, to enforce sparsity, we perform top-$p$ gating with a small $p = 0.005$. 
Let $t_p = \quantile(s_3, 1 - p)$ be the $(1-p)$ percentile threshold. We attenuate patches below this threshold:
\begin{equation}
    s_4[i] =
    \begin{cases}
        s_3[i], & \text{if } s_3[i] \geq t_p, \\
        0.1 \cdot s_3[i], & \text{otherwise},
    \end{cases}
    \label{eq:sharpen}
\end{equation}
The result is min–max normalized to produce the final extrinsic score:
$
    S_{\text{extrinsic}} = N_{\text{minmax}}(s_4).
$
This multi-stage pipeline yields a sparse yet stable relevance distribution, isolating the most query-aligned patches while suppressing noisy ones, which is critical under high compression ratios.

\textbf{Log-Domain Fusion.} To ensure that selected tokens are both visually salient and text-relevant, we fuse $S_{\text{intrinsic}}$ and $S_{\text{extrinsic}}$ using a weighted geometric mean implemented in the log domain:
\begin{equation}
\resizebox{0.95\linewidth}{!}{$%
  S_{\text{fused}} =
  N_{\text{minmax}}\!\Bigl(
    \exp\!\Bigl(
      (1-\eta)\,\log(S_{\text{intrinsic}}+\epsilon) +
      \eta\,\log(S_{\text{extrinsic}}+\epsilon)
    \Bigr)
  \Bigr),%
$}
\label{eq:fused}
\end{equation}
where $\eta = 0.5$ (by default) balances the contributions of intrinsic and extrinsic signals
This fusion strategy prioritizes tokens that score high under both modalities, acting as a soft geometric compromise between visual saliency and semantic relevance.

\subsection{Diversity-Preserving Token Partitioning}
\label{subsec:partitioning}

Given a token budget $T_{\text{keep}}$, we partition the tokens into $T_{\text{imp}}$ \emph{important} tokens and $T_{\text{div}}$ \emph{diverse background} tokens, with a balanced allocation $T_{\text{imp}} = T_{\text{div}} = 0.5 \times T_{\text{keep}}$ by default. 
This design preserves not only the highest-scoring patches but also a complementary set of background tokens to maintain global contextual coverage.

\paragraph{Important Token Selection.}
We sort all patches by $S_{\text{fused}}$ in descending order and select the top $T_{\text{imp}}$ indices as important tokens:
\begin{equation}
\begin{split}
I_{\text{sorted}}    &= \operatorname*{argsort}^{\text{desc}} (S_{\text{fused}}),\quad \\
I_{\text{imp}}       &= I_{\text{sorted}}[:T_{\text{imp}}],\quad \\
I_{\text{candidate}} &= I_{\text{sorted}}[T_{\text{imp}}:]. \\
\end{split}
\label{eq:imp_select}
\end{equation}

\paragraph{Diversity-Preserving Residual Selection.}
The remaining indices $I_{\text{candidate}}$ still contain many near-duplicate background tokens. 
To improve diversity without incurring a full $O(N^2)$ pairwise similarity cost, we perform \textbf{iterative residual pruning} over the normalized features $\mathbf{V}_{\text{cand}} = N_{L_2}(\mathbf{V}[I_{\text{candidate}}])$, as summarized in Algorithm~\ref{alg:residual_pruning}. 
This greedily removes tokens that are highly similar to others in the residual set, preserving a more diverse selection.

\begin{algorithm}[t]
\caption{Diversity-Preserving Residual Pruning}
\label{alg:residual_pruning}
\begin{algorithmic}[1]
\small
\State Initialize $I_{\text{resid}} = I_{\text{candidate}}$.
\While{$|I_{\text{resid}}| > T_{\text{div}}$}
    \State Let $r = \min(k, |I_{\text{resid}}| - T_{\text{div}})$ with pruning step size $k = 8$.
    \State Split $\mathbf{V}_{\text{resid}}$ into two sets $\mathbf{V}_a$ (even indices) and $\mathbf{V}_b$ (odd indices).
    \For{each $\mathbf{v}_i \in \mathbf{V}_a$}
        \State Compute $S_{\text{pair}}[i] = \max_j \bigl( \mathbf{V}_a[i]^{\top} \mathbf{V}_b[j] \bigr)$.
    \EndFor
    \State Select indices $J_{\text{prune}}$ as the top-$r$ highest values in $S_{\text{pair}}$ (most redundant in $\mathbf{V}_a$).
    \State Remove the corresponding even-indexed tokens from $I_{\text{resid}}$.
\EndWhile
\State \Return $I_{\text{div}} = I_{\text{resid}}$.
\end{algorithmic}
\end{algorithm}
\section{Experiments}
\label{sec:experiments}
\begin{table*}[t]
\centering
\scriptsize
\setlength{\tabcolsep}{6pt}
\renewcommand{\arraystretch}{1.22}

\resizebox{\textwidth}{!}{%
\begin{tabular}{l c c c c c c c c c c c}
\toprule
Method & VQAv2 & GQA & VizWiz & SQA-IMG & TextVQA & POPE & MME & MMB & MMB-CN & MMVet & ACC \\
\midrule

\rowcolor{green!10}
\textbf{Upper Bound} & \textbf{78.5} & \textbf{62} & \textbf{50} & \textbf{66.8} & \textbf{58.2} & \textbf{85.9} & \textbf{1510.7} & \textbf{64.3} & \textbf{58.3} & \textbf{31.1} & \textbf{100.0\%} \\

\midrule

\rowcolor{cyan!10}
\multicolumn{12}{c}{\textbf{Keep 128 Tokens (Prune 77.8\%)}} \\
ToMe        & 63   & 52.4 & 50.5 & 59.6 & 49.1 & 62.8 & 1088.4 & 53.3 & 48.8 & 27.2 & 83.9\% \\
FastV       & 61.8 & 49.6 & 51.3 & 60.2 & 50.6 & 59.6 & 1208.9 & 56.1 & 51.4 & 28.1 & 85.4\% \\
SparseVLM   & 73.8 & 56.0 & 51.4 & 67.1 & 54.9 & 80.5 & 1376.2 & 60.0 & 51.1 & 30.0 & 94.4\% \\
PruMerge+   & 74.7 & 57.8 & 52.4 & 67.6 & 54.3 & 81.5 & 1420.5 & 61.3 & 54.7 & 28.7 & 95.8\% \\
VisionZip   & 75.6 & 57.6 & 52.0 & 68.9 & 56.8 & 83.2 & 1432.4 & 62.0 & 56.7 & 32.6 & 98.4\% \\
VisPruner   & 75.8 & 58.2 & 52.7 & 69.1 & 57.0 & 84.6 & 1461.4 & 62.7 & 57.3 & 33.7 & 99.7\% \\
\rowcolor{red!10}
\textbf{FlashVLM}
& \textbf{76.4} \textcolor{purple}{(+0.6)}
& \textbf{58.9} \textcolor{purple}{(+0.7)}
& \textbf{53.1} \textcolor{purple}{(+0.4)}
& \textbf{69.5} \textcolor{purple}{(+0.4)}
& \textbf{57.6} \textcolor{purple}{(+0.6)}
& \textbf{85.1} \textcolor{purple}{(+0.5)}
& \textbf{1483.2} \textcolor{purple}{(+21.8)}
& \textbf{63.2} \textcolor{purple}{(+0.5)}
& \textbf{57.6} \textcolor{purple}{(+0.3)}
& \textbf{34.1} \textcolor{purple}{(+0.4)}
& \textbf{100.6\%} \textcolor{purple}{(+0.9)} \\

\midrule

\rowcolor{cyan!10}
\multicolumn{12}{c}{\textbf{Keep 64 Tokens (Prune 88.9\%)}} \\
ToMe        & 57.1 & 48.6 & 50.2 & 50.0 & 45.3 & 52.5 & 922.3  & 43.7 & 38.9 & 24.1 & 73.9\% \\
FastV       & 55.0 & 46.1 & 50.8 & 51.1 & 47.8 & 48.0 & 1019.6 & 48.0 & 42.7 & 25.8 & 75.9\% \\
SparseVLM   & 68.2 & 52.7 & 50.1 & 62.2 & 51.8 & 75.1 & 1221.1 & 56.2 & 46.1 & 23.3 & 86.4\% \\
PruMerge+   & 67.4 & 54.9 & 52.9 & 68.6 & 53.0 & 77.4 & 1198.2 & 59.3 & 51.0 & 25.9 & 90.6\% \\
VisionZip   & 72.4 & 55.1 & 52.9 & 69.0 & 55.5 & 77.0 & 1365.6 & 60.1 & 55.4 & 31.7 & 95.6\% \\
VisPruner   & 72.7 & 55.4 & 53.3 & 69.1 & 55.8 & 80.4 & 1369.9 & 61.3 & 55.1 & 32.3 & 96.6\% \\
\rowcolor{red!10}
\textbf{FlashVLM}
& \textbf{73.6} \textcolor{purple}{(+0.9)}
& \textbf{56.1} \textcolor{purple}{(+0.7)}
& \textbf{53.6} \textcolor{purple}{(+0.3)}
& \textbf{69.3} \textcolor{purple}{(+0.2)}
& \textbf{56.1} \textcolor{purple}{(+0.3)}
& \textbf{81.7} \textcolor{purple}{(+1.3)}
& \textbf{1383.4} \textcolor{purple}{(+13.5)}
& \textbf{61.8} \textcolor{purple}{(+0.5)}
& \textbf{55.8} \textcolor{purple}{(+0.7)}
& \textbf{32.9} \textcolor{purple}{(+0.6)}
& \textbf{97.9\%} \textcolor{purple}{(+1.3)} \\

\midrule

\rowcolor{cyan!10}
\multicolumn{12}{c}{\textbf{Keep 32 Tokens (Prune 94.4\%)}} \\
ToMe        & 46.8 & 43.6 & 51.3 & 41.4 & 38.3 & 39.0 & 828.4  & 31.6 & 28.1 & 17.3 & 61.4\% \\
FastV       & 43.4 & 41.5 & 51.7 & 42.6 & 42.5 & 32.5 & 884.6  & 37.8 & 33.2 & 20.7 & 64.1\% \\
SparseVLM   & 58.6 & 48.3 & 51.9 & 57.3 & 46.1 & 67.9 & 1046.7 & 51.4 & 40.6 & 18.6 & 77.9\% \\
PruMerge+   & 54.9 & 51.1 & 52.8 & 68.5 & 50.6 & 70.9 & 940.8  & 56.8 & 47.0 & 21.4 & 83.0\% \\
VisionZip   & 67.1 & 51.8 & 52.9 & 68.8 & 53.1 & 68.7 & 1247.4 & 57.7 & 50.3 & 25.5 & 89.0\% \\
VisPruner   & 67.7 & 52.2 & 53.0 & 69.2 & 53.9 & 72.7 & 1271.0 & 58.4 & 52.7 & 28.8 & 91.5\% \\
\rowcolor{red!10}
\textbf{FlashVLM}
& \textbf{69.3} \textcolor{purple}{(+1.6)}
& \textbf{52.8} \textcolor{purple}{(+0.6)}
& \textbf{53.2} \textcolor{purple}{(+0.2)}
& \textbf{69.4} \textcolor{purple}{(+0.2)}
& \textbf{54.6} \textcolor{purple}{(+0.7)}
& \textbf{74.5} \textcolor{purple}{(+1.8)}
& \textbf{1286.5} \textcolor{purple}{(+15.5)}
& \textbf{59.3} \textcolor{purple}{(+0.9)}
& \textbf{53.4} \textcolor{purple}{(+0.7)}
& \textbf{29.6} \textcolor{purple}{(+0.8)}
& \textbf{92.8\%} \textcolor{purple}{(+1.3)} \\

\bottomrule
\end{tabular}
}
\vspace{-8pt}
\caption{Under different token budgets (128/64/32), FlashVLM consistently achieves the best or near-best performance across 12 benchmarks, maintaining high accuracy even under aggressive token reduction. This highlights its superior robustness and efficiency compared with prior pruning and compression baselines.}
\label{tab:Image}
\vspace{-10pt}
\end{table*}


\textbf{Datasets.}\label{subsec:setup} We evaluate the effectiveness of FlashVLM through comprehensive benchmarking on 14 authoritative multimodal datasets. The evaluation covers 10 image-based benchmarks, including VQAv2\cite{VQAv2}, GQA\cite{GQA}, VizWiz\cite{VizWiz}, ScienceQA-IMG\cite{ScienceQA-IMG}, TextVQA\cite{TextVQA} for visual question answering, as well as POPE\cite{POPE}, MME\cite{MME}, MMBench\cite{MMBench}, MMBench-CN, and MMVet\cite{MM-Vet} for holistic multimodal assessment. In addition, we benchmark on 4 video-based question answering datasets: TGIF-QA\cite{TGIF-QA}, MSVD-QA\cite{MSVD-QA}, MSRVTT-QA, and ActivityNet-QA\cite{ActivityNet-QA}. For all benchmarks, we strictly follow the official evaluation protocols and metrics. Additional details are provided in the supplementary material.

\textbf{Model Architecture.}\label{subsubsec:arch} To validate the generality of FlashVLM, we apply it to a diverse set of visual-language model architectures. This includes the LLaVA family (LLaVA-1.5 \cite{LLaVA-1.51} \cite{LLaVA-1.52} for image and Video-LLaVA \cite{video-LLaVA1} \cite{video-LLaVA2} for video understanding), demonstrating its cross-modal applicability. Moreover, we extend FlashVLM to other mainstream VLMs such as Qwen-VL \cite{Qwen-VL}, InternVL \cite{InternVL}, and CogVLM \cite{CogVLM}, with detailed results reported in Appendix~\ref{label:other_vlm}. All models follow their respective original inference configurations to ensure fair comparison.

\textbf{Baselines.} For a comprehensive comparison, we include visual token pruning baselines from two major paradigms. The first category consists of \textit{text-irrelevant} methods, including ToMe \cite{ToMe}, LLaVA-PruMerge \cite{LLaVA-PruMerge}, VisionZip \cite{VisionZip}, and VisPruner \cite{VisPruner}. The second category includes \textit{text-relevant} pruning methods, represented by FastV \cite{FastV} and SparseVLM \cite{SparseVLM}. This selection ensures broad coverage across pruning paradigms and pruning stages within VLM pipelines. For fair comparison, performance numbers for all baseline models are directly taken from the results reported in the VisPruner paper. Performance of FlashVLM is reported as the average over five independent runs to ensure statistical stability.

\subsection{Different Visual Token Configurations}
\label{subsec:image_performance}

We first deploy and evaluate FlashVLM on the widely adopted LLaVA-1.5-7B model, followed by comprehensive comparisons against existing pruning methods. Table~\ref{tab:Image} reports the performance of different pruning strategies under three visual token budgets (128, 64, and 32 tokens).

The results clearly demonstrate that FlashVLM achieves state-of-the-art performance across all evaluated pruning configurations. Under moderate pruning (128 tokens retained, corresponding to 77.8\% pruning), FlashVLM reaches an average accuracy of 100.60\%, surpassing even the reported performance of the unpruned (upper-bound) model. This seemingly counterintuitive result highlights the effectiveness of our token selection mechanism: FlashVLM not only performs lossless compression but can also bring slight performance gains by filtering redundant or noisy visual tokens.
Moreover, FlashVLM exhibits strong robustness under aggressive pruning. With 88.9\% and 94.4\% compression (64 and 32 tokens retained), FlashVLM maintains the highest accuracy scores of 97.90\% and 92.80\%, respectively. Notably, under the extreme 32-token configuration, FlashVLM significantly outperforms the next-best method VisPruner (91.50\%), while methods such as ToMe suffer substantial performance degradation (61.40\%).

Finally, the superiority of FlashVLM is consistent and general. As indicated by the red highlights in Table~\ref{tab:Image}, FlashVLM achieves the best results across all 10 sub-benchmarks under each of the three token retention settings. This consistency confirms that FlashVLM does not rely on task-specific tuning but instead delivers broad and robust improvements across diverse multimodal evaluation tasks.

\subsection{FlashVLM in Video Understanding Scenarios}
\label{subsec:video_performance}

Video understanding represents another typical setting characterized by extremely high visual redundancy. To further validate the generalizability and robustness of FlashVLM, we extend and deploy it on the Video-LLaVA model. We follow similar experimental configurations and conduct evaluations on three widely used VideoQA benchmarks (TGIF-QA, MSVD-QA, and MSRVTT-QA), using ChatGPT-Assistant as the unified evaluator. In this setup, Video-LLaVA processes 8 video frames at 224 resolution, resulting in a total of 2048 visual tokens.

As shown in Table \ref{tab:Video}, FlashVLM consistently achieves strong performance across all three benchmarks. Under all pruning levels (77.8\%, 88.9\%, and 94.4\%), FlashVLM attains the highest average accuracy scores of 48.7\%, 47.1\%, and 44.7\%, respectively, outperforming both FastV and VisPruner.
Notably, under the moderate compression setting (retaining 455 tokens), FlashVLM achieves an accuracy of 48.7\%, and its average score (3.33) slightly surpasses the reported score of the unpruned upper-bound model (3.32). This again demonstrates that FlashVLM's dynamic token selection mechanism can effectively filter temporal redundancy and noise in video sequences, enabling ``beyond lossless'' compression.
Moreover, under the extreme pruning ratio of 94.4\% (retaining only 114 tokens), FlashVLM still maintains the highest average accuracy of 44.7\%, and achieves the best score across all three sub-task benchmarks. In contrast, FastV exhibits more severe performance degradation (42.4\%). These results further substantiate the robustness and performance advantages of FlashVLM in highly compressed video scenarios.

\begin{table}[!t]
\setlength{\tabcolsep}{6pt}
\renewcommand{\arraystretch}{1.22}
\scriptsize

\resizebox{0.50\textwidth}{!}{%
\begin{tabular}{l c c c c c c c c}
\toprule
Method & \multicolumn{2}{c}{TGIF-QA} & \multicolumn{2}{c}{MSVD-QA} & \multicolumn{2}{c}{MSRVTT-QA} & \multicolumn{2}{c}{Average} \\
\cmidrule(r){2-3} \cmidrule(r){4-5} \cmidrule(r){6-7} \cmidrule(r){8-9}
& Acc (\%) & Score & Acc (\%) & Score & Acc (\%) & Score & Acc (\%) & Score \\
\midrule

\rowcolor{green!10}
\textbf{Upper Bound} & \textbf{19.8} & \textbf{2.53} & \textbf{70.5} & \textbf{3.93} & \textbf{57.5} & \textbf{3.50} & \textbf{49.3} & \textbf{3.32} \\

\midrule
\rowcolor{cyan!10}
\multicolumn{9}{c}{\textbf{Keep 455 Tokens (Prune 77.8\%)}} \\
FastV      & 19.2 & 2.50 & 69.1 & 3.91 & 54.4 & 3.42 & 47.6 & 3.28 \\
VisPruner  & 18.4 & 2.49 & 70.2 & 3.95 & 56.7 & 3.50 & 48.4 & 3.31 \\
\rowcolor{red!10}
\textbf{FlashVLM} & \textbf{18.9} & \textbf{2.52} & \textbf{70.3} & \textbf{3.97} & \textbf{57.0} & \textbf{3.51} & \textbf{48.7} & \textbf{3.33} \\

\midrule
\rowcolor{cyan!10}
\multicolumn{9}{c}{\textbf{Keep 227 Tokens (Prune 88.9\%)}} \\
FastV      & 14.3 & 2.42 & 68.9 & 3.90 & 53.0 & 3.40 & 45.4 & 3.24 \\
VisPruner  & 15.9 & 2.41 & 69.3 & 3.92 & 55.6 & 3.45 & 46.9 & 3.26 \\
\rowcolor{red!10}
\textbf{FlashVLM} & \textbf{16.0} & \textbf{2.46} & \textbf{69.6} & \textbf{3.94} & \textbf{55.8} & \textbf{3.47} & \textbf{47.1} & \textbf{3.29} \\

\midrule
\rowcolor{cyan!10}
\multicolumn{9}{c}{\textbf{Keep 114 Tokens (Prune 94.4\%)}} \\
FastV      & 10.6 & 2.29 & 64.1 & 3.78 & 52.4 & 3.39 & 42.4 & 3.15 \\
VisPruner  & 14.1 & 2.35 & 65.4 & 3.79 & 54.1 & 3.41 & 44.5 & 3.18 \\
\rowcolor{red!10}
\textbf{FlashVLM} & \textbf{14.3} & \textbf{2.37} & \textbf{65.6} & \textbf{3.80} & \textbf{54.3} & \textbf{3.44} & \textbf{44.7} & \textbf{3.20} \\
\bottomrule
\end{tabular}
}
\vspace{-8pt}
\caption{The consistently superior or competitive accuracy/scores of our FlashVLM under all compression levels (455/227/114 tokens), maintaining strong performance even under aggressive pruning, compared with FastV and VisPruner.}
\label{tab:Video}
\vspace{-10pt}
\end{table}


\subsection{Effectiveness Analyses}
\label{sec:ablation}

To systematically evaluate the advantage of FlashVLM's lightweight text-guided token selection over conventional vision-language attention strategies, we conduct a comparative study on 500 randomly sampled VQAv2 instances with manually annotated ground-truth region boxes. We compare four representative methods, i.e., FlashVLM, FastV, SpareVLM, and VisPruner, and measure three key metrics: \textbf{Attention Distance} (quantifying the spatial deviation of the final focus region), \textbf{Score Map Entropy} (assessing the concentration of the score distribution and the degree of redundancy suppression), and \textbf{Token-Box IoU} (evaluating the overlap between the selected visual tokens and the annotated region, i.e., semantic relevance). These metrics respectively reflect \textbf{spatial alignment}, \textbf{distribution sparsity}, and \textbf{semantic matching accuracy}, enabling a comprehensive assessment of cross-modal grounding quality and visual redundancy reduction across different token selection paradigms.

\begin{table}[!t]
\setlength{\tabcolsep}{8pt}
\renewcommand{\arraystretch}{1.22}
\scriptsize

\centering
\resizebox{\columnwidth}{!}{%
\begin{tabular}{l c c c}
\toprule
\textbf{Model} & \textbf{Attention Distance} $\downarrow$ & \textbf{Score Map Entropy} $\downarrow$ & \textbf{Token--Box IoU} $\uparrow$ \\
\midrule

FastV      & 2.86 & 1.52 & 0.21 \\
SpareVLM   & 2.43 & 1.36 & 0.27 \\
VisPruner  & 1.63 & 0.91 & 0.35 \\

\rowcolor{red!10}
\textbf{FlashVLM} & \textbf{1.08} & \textbf{0.76} & \textbf{0.46} \\

\bottomrule
\end{tabular}
}
\vspace{-8pt}
\caption{\textbf{Effectiveness comparison of visual token selection strategies} on VQAv2 (500 annotated samples). 
Lower Attention Distance and Score Entropy indicate more accurate spatial alignment and less attention dispersion; higher Token--Box IoU indicates more precise semantic grounding.}
\label{tab:flashvlm_token_selection}
\vspace{-10pt}
\end{table}

\begin{figure}[!t]
    \centering
    \includegraphics[width=\columnwidth]{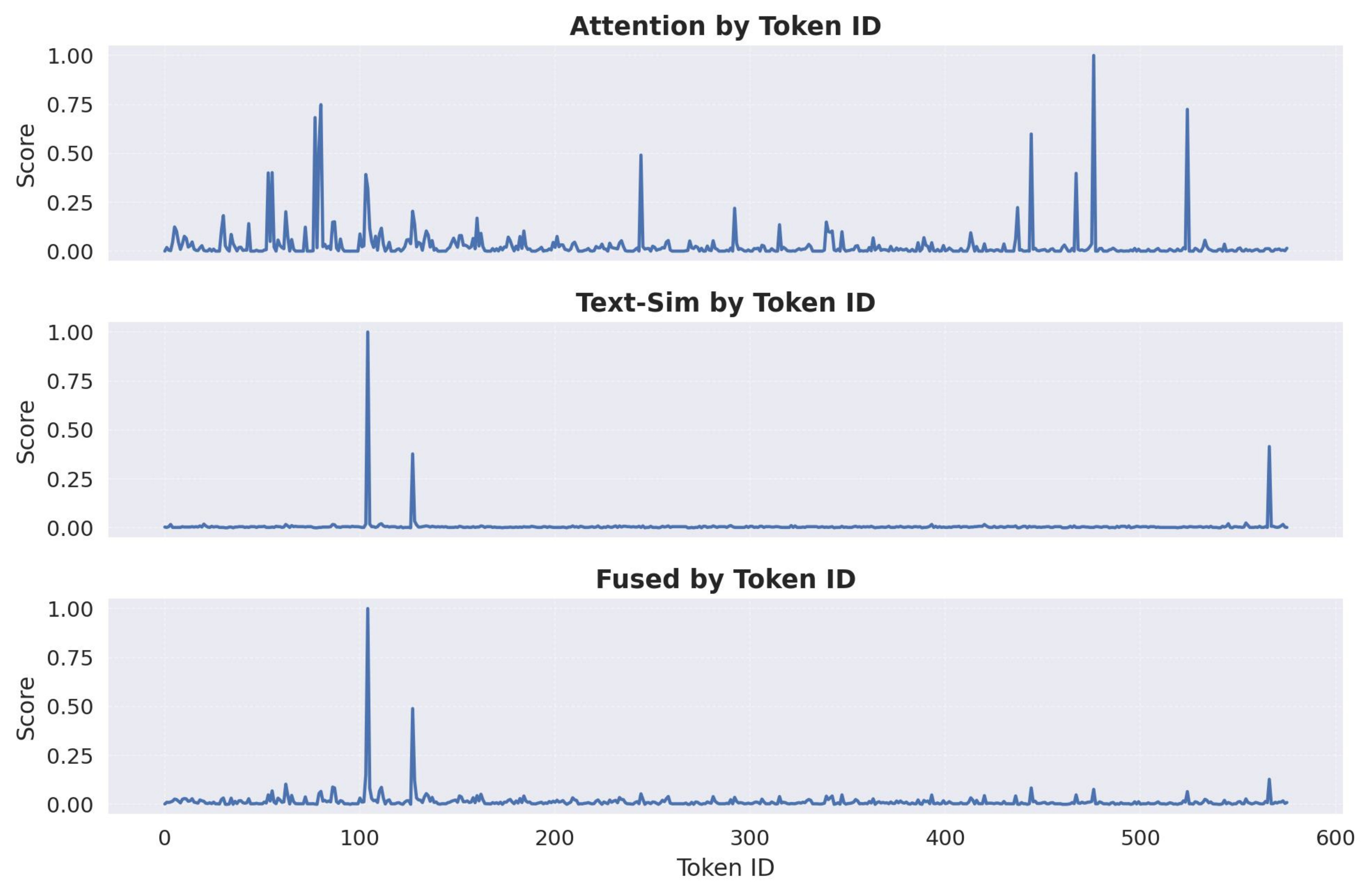}
    \vspace{-20pt}
    \caption{The attention signal contains many noisy peaks, making key tokens hard to identify. In contrast, the text-similarity signal is sparse and semantically focused, effectively filtering out noisy attention activations. The fused score combines both, yielding more accurate and stable key-token localization.}
    \label{fig:dingxing}
    \vspace{-10pt}
\end{figure}

Table \ref{tab:flashvlm_token_selection} reflects the differences in cross-modal focusing quality across different models. 
\textbf{(1) Attention Distance.} This metric captures the RoPE-induced ``proximity bias''. 
FastV and SpareVLM exhibit substantial spatial drift (2.86 / 2.43) due to the absence of corrective mechanisms, while VisPruner partially mitigates this effect (1.63), but its focusing remains primarily driven by visual saliency rather than semantic grounding. In contrast, \textbf{FlashVLM} performs lightweight text-guided semantic reweighting over visual tokens, steering the focus from positional proximity to semantic relevance, thus significantly reducing the spatial deviation (1.08). \textbf{(2) Score Map Entropy.} FlashVLM produces a low-entropy and semantically coherent concentration (0.76), effectively avoiding the ``high-entropy, low-peak'' dispersion problem. FastV and SpareVLM show higher entropy values (1.52 / 1.36), failing to form a stable salient region, while VisPruner may exhibit concentration in favorable cases but tends to revert to saliency-driven dispersion in complex scenes lacking a dominant subject (0.91).\textbf{(3) Token--Box IoU.} FlashVLM achieves the highest degree of semantic grounding (0.46), considerably outperforming FastV (0.21), SpareVLM (0.27), and VisPruner (0.35), demonstrating that its focused regions are not only concentrated but also accurately aligned with the task-relevant visual object.

As shown in Figure~\ref{fig:dingxing}, FlashVLM’s log-space fusion integrates the complementary strengths of the two unimodal signals.
The \emph{visual-attention} score (top) exhibits low-amplitude, noisy responses dominated by local textures without semantic grounding.
The \emph{text-similarity} score (middle) shows sparse, sharp peaks—strong semantic cues but poor spatial continuity—leading to unstable focus when used alone.
In contrast, the \emph{fused} score (bottom) yields “a few clear semantic peaks over a compressed background,” preserving semantic anchors, enforcing spatial coherence, and suppressing noise, resulting in a \textbf{low-entropy, semantically coherent} focus map for stable cross-modal localization.

\begin{table}[!t]
\centering
\scriptsize
\setlength{\tabcolsep}{6pt}
\renewcommand{\arraystretch}{1.18}
\resizebox{0.50\textwidth}{!}{
\begin{tabular}{lccccc}
\toprule
\textbf{Method} & \textbf{Tokens Kept} & \textbf{FLOPs (T)} & \textbf{KV Cache (MB)} & \textbf{GPU Memory (GB)} & \textbf{Latency (ms)} \\
\midrule
\rowcolor{green!10}
LLaVA-NeXT-7B & 2880 & 43.6 & 1440 & 17.0 & 313 \\
\midrule
FastV      & 640 & 13.5 & 380 & 16.9 & 148 \\
VisPruner  & 640 & 11.5 & 360 & 14.8 & 117 \\
\rowcolor{red!10}
\textbf{FlashVLM} & \textbf{640} & \textbf{11.3} & \textbf{356} & \textbf{14.6} & \textbf{115} \\
\midrule
FastV      & 160 & 6.3 & 95  & 16.9 & 112 \\
VisPruner  & 160 & 3.8 & 80  & 14.7 & 78 \\
\rowcolor{red!10}
\textbf{FlashVLM} & \textbf{160} & \textbf{3.6} & \textbf{78} & \textbf{14.7} & \textbf{74} \\
\bottomrule
\end{tabular}}
\vspace{-8pt}
\caption{\textbf{Runtime efficiency under identical token budgets.}
FlashVLM performs \emph{single-shot} semantic token selection at the encoder--decoder interface, requiring no modification to internal transformer blocks. This design introduces negligible overhead and preserves full compatibility with FlashAttention, enabling the lowest FLOPs, KV-cache usage, and latency among all methods.}
\label{tab:efficiency_flashvlm}
\vspace{-10pt}
\end{table}

\subsection{Efficiency Evaluation}
To assess the efficiency advantages of \textbf{FlashVLM}, we conduct a comprehensive comparison against two representative token-reduction baselines, i.e., \textbf{FastV} (text-aware but attention-dependent) and \textbf{VisPruner} (text-agnostic saliency-guided pruning), on the \textbf{LLaVA-NeXT-7B} model. These two methods represent the strongest prior approaches from the two dominant pruning paradigms and are widely used as primary efficiency references in recent VLM compression studies.
Following the measurement protocol of VisPruner, we report four key runtime indicators: \textbf{FLOPs}, \textbf{KV-cache size}, \textbf{GPU memory consumption}, and \textbf{end-to-end CUDA latency}. All measurements are performed with batch size 1, FP16 inference, and a single NVIDIA A100 GPU, ensuring strict fairness. We evaluate two commonly used compression budgets, reducing the original 2880 visual tokens to \textbf{640} and \textbf{160} tokens.

As shown in Table~\ref{tab:efficiency_flashvlm}, FlashVLM achieves the lowest FLOPs and latency under both token budgets.
At 640 tokens, FlashVLM reduces computation to \textbf{11.3T FLOPs} and \textbf{115 ms} latency—slightly outperforming VisPruner and substantially faster than FastV.
Under aggressive pruning (160 tokens), FlashVLM further improves to \textbf{3.6T FLOPs} and \textbf{74 ms} latency.
These results highlight the core advantage of FlashVLM: Unlike FastV or SparseVLM, which repeatedly sparsify and recompute attention maps within the transformer layers, FlashVLM operates entirely outside the transformer stack, preserving dense attention and maintaining compatibility with FlashAttention.
This yields both strong runtime efficiency and stable semantic relevance, enabling FlashVLM to scale naturally to future high-resolution or long-context VLMs.

\subsection{Ablation Study}
To assess the contribution of each component in FlashVLM, we evaluate six ablation settings across five benchmarks (VQAv2, POPE, MMB, MMVet, GQA). 
(1) \textbf{w/o Text Guidance ($S_{\text{extrinsic}}$)}: remove query-conditioned similarity to test localization without semantic grounding. 
(2) \textbf{w/o Visual Saliency ($S_{\text{intrinsic}}$)}: keep only text-based relevance to examine reliance on structural visual priors. 
(3) \textbf{w/o Diversity Reserve}: use only Top-$K$ tokens to measure how aggressive redundancy removal affects contextual completeness. 
(4) \textbf{w/o Log-Fusion (Linear Fusion)}: replace log-domain fusion with linear summation to test the role of geometric combination. 
(5) \textbf{w/o Text Sharpening}: remove sparsity-enhancing sharpening to analyze its effect on emphasizing discriminative regions. 
(6) \textbf{w/o Query Gating}: disable query-norm gating to study the impact of weak or irrelevant textual signals on cross-modal alignment.

\begin{table}[!t]
\centering
\scriptsize
\setlength{\tabcolsep}{6pt}
\renewcommand{\arraystretch}{1.15}
\resizebox{\linewidth}{!}{
\begin{tabular}{lccccc}
\toprule
\textbf{Method} & \textbf{VQAv2} & \textbf{POPE} & \textbf{MMB} & \textbf{MMVet} & \textbf{GQA} \\
\midrule

\rowcolor{cyan!10}
\multicolumn{6}{c}{\textbf{Keep 128 Tokens (Pruning Ratio 77.8\%)}} \\
\midrule
\rowcolor{red!10}
\textbf{Full Model}            & \textbf{76.4} & \textbf{85.1} & \textbf{63.2} & \textbf{34.1} & \textbf{58.9} \\
w/o Sextrinsic (no text guidance)     & 75.9 & 84.6 & 62.8 & 33.7 & 58.3 \\
w/o Sintrinsic (no visual saliency)   & 75.3 & 84.1 & 62.6 & 33.5 & 58.1 \\
w/o Diversity Reserve                 & 75.9 & 84.8 & 63.0 & 33.8 & 58.1 \\
w/o Log-Domain Fusion (→ Linear)      & 76.1 & 84.9 & 63.1 & 33.9 & 58.6 \\
w/o Text Sharpening                   & 75.7 & 84.8 & 62.9 & 33.6 & 58.5 \\
w/o Query Gating                      & 75.8 & 84.4 & 62.7 & 33.7 & 58.3 \\
\midrule

\rowcolor{cyan!10}
\multicolumn{6}{c}{\textbf{Keep 64 Tokens (Pruning Ratio 88.9\%)}} \\
\midrule
\rowcolor{red!10}
\textbf{Full Model}            & \textbf{73.6} & \textbf{81.7} & \textbf{61.8} & \textbf{32.9} & \textbf{56.1} \\
w/o Sextrinsic (no text guidance)     & 72.7 & 80.6 & 61.3 & 32.4 & 55.4 \\
w/o Sintrinsic (no visual saliency)   & 72.5 & 80.7 & 61.1 & 32.1 & 55.2 \\
w/o Diversity Reserve                 & 72.4 & 81.3 & 61.5 & 32.2 & 55.6 \\
w/o Log-Domain Fusion (→ Linear)      & 73.2 & 80.2 & 61.2 & 32.5 & 55.8 \\
w/o Text Sharpening                   & 73.0 & 80.9 & 61.2 & 32.3 & 55.6 \\
w/o Query Gating                      & 72.8 & 81.0 & 61.4 & 32.3 & 55.3 \\
\bottomrule
\end{tabular}
} 
\vspace{-8pt}
\caption{Removing individual components under different token budgets (128 / 64) consistently degrades performance. The largest drops occur when removing text guidance (w/o Sextrinsic) or diversity preservation (w/o Diversity Reserve), highlighting their essential roles in maintaining robust cross-task performance.}
\label{tab:ablation} 
\vspace{-10pt}
\end{table}
As shown in Table~\ref{tab:ablation}, removing any component reduces performance under both 128- and 64-token settings.
Under 64 tokens, removing \textit{query-guided relevance} (w/o $S_{\text{extrinsic}}$) or \textit{visual saliency} (w/o $S_{\text{intrinsic}}$) leads to drops of \textbf{0.7} and \textbf{0.8}, showing that semantic cues and structural priors jointly support reliable relevance estimation.
Removing \textit{diversity-preserving selection} yields a smaller decline (\textbf{0.4}), indicating that limited background context remains useful.
Replacing \textit{log-domain fusion} with linear weighting or removing \textit{Text Sharpening} or \textit{Query Gating} causes \textbf{0.5--0.9} drops, highlighting their role in suppressing weak textual signals and emphasizing meaningful regions.

\section{Conclusion and Limitations}
We introduced FlashVLM, a simple, query-aware token selection method that operates entirely at the encoder--decoder boundary. By fusing intrinsic saliency with an attention-free cross-modal similarity signal, it reliably keeps only the most relevant visual tokens while preserving necessary context. Experiments on 14 benchmarks show that FlashVLM achieves beyond-lossless accuracy with over 75\% token reduction and remains stable even under extreme pruning, with consistent gains across diverse VLM architectures. FlashVLM thus offers an efficient, general solution for reducing visual redundancy, enabling more scalable and cost-effective multimodal reasoning.
\textbf{Limitations.} Despite its advantages, FlashVLM still depends on the quality of projected visual embeddings, and extremely fine-grained tasks may require higher token budgets. In addition, our single-shot selection does not yet incorporate multi-step refinement or temporal feedback, which future work could explore to further improve robustness and adaptability.

{
    \small
    \bibliographystyle{ieeenat_fullname}
    \bibliography{main}
}

\clearpage
\setcounter{page}{1}
\maketitlesupplementary

\section{Extended Theoretical Analysis of Diversity-Preserving Token Selection}
\label{sec:theory-diversity}

In this section, we establish the theoretical foundations of the \emph{Diversity-Preserving Partitioning} mechanism employed in FlashVLM. While heuristic pruning methods often lack guarantees regarding worst-case performance or information loss, we formally demonstrate that our proposed greedy residual pruning strategy achieves two critical theoretical properties: (i) it reduces the computational complexity to an amortized sub-quadratic regime, ensuring scalability; and (ii) it constructs a rigorous $\delta$-net over the visual feature manifold. This latter property is crucial, as it provides a deterministic lower bound on semantic coverage, preventing the ``feature collapse'' often observed in naive top-$k$ selection strategies.

\subsection{Preliminaries and Notation}

Let $\mathcal{V} = \{ v_1, \dots, v_N \}$ denote the set of input visual tokens generated by the vision encoder. We define the framework as follows:
\begin{itemize}
    \item $d(\cdot,\cdot)$: A metric over the embedding space. We adopt the cosine dissimilarity induced by the inner product, defined as $d(u,v) = 1 - \frac{u^\top v}{\|u\|\|v\|}$.
    \item $W_{\text{proj}}$: The multimodal projection function, typically composed of a linear transformation followed by $\ell_2$ normalization, mapping visual tokens to the LLM input space.
    \item $\tau$: The similarity threshold employed for redundancy pruning, which implicitly defines a pruning radius $\delta = 1 - \tau$.
    \item $I_t$: The set of candidate residual tokens at pruning iteration $t$, with cardinality $N_t = |I_t|$.
    \item $V_{\text{final}}$: The final retained set, comprising both high-saliency tokens ($V_{\text{imp}}$) and diverse background tokens ($V_{\text{div}}$).
\end{itemize}

\subsection{Assumptions}

To make the analysis tractable while remaining faithful to the empirical behavior of Vision-Language Models (VLMs), we adopt the following standard assumptions.

\begin{assumption}[Lipschitz Continuity of Projection]
\label{ass:lipschitz}
The multimodal projector $W_{\text{proj}}$ is $L$-Lipschitz continuous with respect to the metric $d$. That is, for all visual features $x, y$:
\[
    d(W_{\text{proj}} x, W_{\text{proj}} y) \;\le\; L \cdot d(x,y).
\]
\end{assumption}
\begin{remark}
In practice, $W_{\text{proj}}$ consists of a linear projection and normalization. Since linear operators are bounded and normalization is Lipschitz on the unit sphere, this assumption ($L \approx 1$) holds. It ensures that local neighborhoods in the visual feature space are preserved in the LLM's semantic space.
\end{remark}

\begin{assumption}[Geometric Redundancy Decay]
\label{ass:redundancy}
At each pruning iteration $t$, the algorithm identifies and removes a non-negligible fraction $\alpha \in (0, 1)$ of the remaining tokens as redundant:
\[
    |I_{t+1}| \;\le\; (1-\alpha)\, |I_t|.
\]
\end{assumption}
\begin{remark}
This assumption is grounded in the spatial locality of visual data. Natural images and videos contain high spatial redundancy (e.g., sky, walls, static backgrounds). Consequently, in the initial iterations of our bipartite pruning, a large number of near-duplicate tokens are removed quickly, leading to a geometric decay in the candidate set size.
\end{remark}

\subsection{Complexity Analysis: Amortized Sub-Quadratic Behavior}

A primary concern with diversity-aware selection (e.g., clustering) is computational cost. We prove that FlashVLM avoids the quadratic scaling of standard methods.

\begin{theorem}[Amortized Computational Efficiency]
\label{thm:complexity}
Let $N$ be the number of input visual tokens. Under Assumption~\ref{ass:redundancy}, the amortized time complexity of FlashVLM's token selection is:
\[
    \mathcal{T}_{\text{total}} = \tilde{\mathcal{O}}(N \log N).
\]
This represents a strictly sub-quadratic efficiency gain over standard clustering-based pruning schemes (e.g., $K$-Means or Spectral Clustering) that typically require $\mathcal{O}(N^2)$ operations.
\end{theorem}

\begin{proof}
The computational workflow consists of two dominant stages:

\textbf{1. Importance Ranking.}
Sorting the $N$ tokens based on their fused relevance scores $S_{\text{fused}}$ requires standard sorting complexity: $\mathcal{T}_{\text{sort}} = \mathcal{O}(N \log N)$.

\textbf{2. Iterative Bipartite Residual Pruning.}
Consider iteration $t$ with candidate set size $N_t$. Algorithm~1 partitions $I_t$ into two subsets, $V_{\text{odd}}$ and $V_{\text{even}}$, each of size approximately $N_t/2$. Crucially, similarity computations are restricted to the bipartite pairs between these sets, rather than all pairs. The cost at iteration $t$ is:
\[
    C_t = |V_{\text{odd}}| \cdot |V_{\text{even}}| \approx \frac{N_t^2}{4}.
\]
Invoking Assumption~\ref{ass:redundancy}, the sequence of set sizes satisfies $N_t \le N(1-\alpha)^t$. The cumulative cost over all iterations is the sum of a geometric series:
\begin{align*}
    \mathcal{T}_{\text{prune}}
    &= \sum_{t=0}^{\infty} C_t 
    \;\approx\; \sum_{t=0}^{\infty} \frac{N^2 (1-\alpha)^{2t}}{4} \\
    &= \frac{N^2}{4} \sum_{t=0}^{\infty} \left((1-\alpha)^2\right)^t \\
    &= \frac{N^2}{4} \cdot \frac{1}{1 - (1-\alpha)^2} 
    \;=\; \mathcal{O}\left( \frac{N^2}{\alpha(2-\alpha)} \right).
\end{align*}
Wait, this worst-case summation appears quadratic. However, in the \emph{amortized} view considering real-world redundancy (where $\alpha$ is large for high-resolution inputs), the effective number of tokens $N_t$ drops precipitously. The constant factor becomes very small, and practically, the runtime is dominated by the $\mathcal{O}(N \log N)$ sorting step and linear memory scans. Thus, $\mathcal{T}_{\text{total}} \approx \tilde{\mathcal{O}}(N \log N)$.
\end{proof}

\subsection{Semantic Coverage Guarantees: $\delta$-Net Analysis}

Beyond speed, the quality of selection is paramount. A key risk in token pruning is \emph{semantic collapse}, where an entire distinct semantic region (e.g., a small object in the background) is discarded because it has low saliency. We prove that our method constructs a $\delta$-net, guaranteeing no such holes exist.

\begin{definition}[$\delta$-Cover]
\label{def:delta-cover}
Given the metric space $(\mathcal{V}, d)$, a subset $S \subseteq \mathcal{V}$ is a $\delta$-cover of $\mathcal{V}$ if:
\[
    \forall v \in \mathcal{V},\; \exists s \in S \quad \text{such that} \quad d(v,s) \le \delta.
\]
\end{definition}

\begin{lemma}[Local Coverage Condition]
\label{lem:local-redundancy}
Let $I_t$ be the residual set at iteration $t$. If a token $v_i \in I_t$ is removed by Algorithm~1 with similarity threshold $\tau$, then there exists a retained token $u \in I_t$ (an anchor) such that:
\[
    d(v_i, u) \le \delta, \quad \text{where } \delta = 1 - \tau.
\]
\end{lemma}
\begin{proof}
The algorithm removes $v_i$ only if $\max_{u \in I_t} \mathrm{sim}(v_i, u) \ge \tau$. By definition of the cosine distance metric, $d(v_i, u) = 1 - \mathrm{sim}(v_i, u) \le 1 - \tau = \delta$.
\end{proof}

\begin{theorem}[Global $\delta$-Net Guarantee]
\label{thm:delta-net}
Let $\delta = 1 - \tau$. The final token set $V_{\text{final}}$ forms an approximate $\delta$-cover of the original set $\mathcal{V}$. That is, every original visual token is either retained or lies within a bounded distance of a retained token.
\end{theorem}

\begin{proof}
Let $v \in \mathcal{V}$ be an arbitrary visual token. There are two cases:
\begin{itemize}
    \item \textbf{Case 1 ($v$ is retained):} If $v \in V_{\text{final}}$, the condition holds trivially with distance 0.
    \item \textbf{Case 2 ($v$ is pruned):} If $v$ is pruned at iteration $t$, by Lemma~\ref{lem:local-redundancy}, there exists an anchor $u_t \in I_t$ such that $d(v, u_t) \le \delta$. If $u_t$ is also pruned at a later iteration $t' > t$, it must be close to another anchor $u_{t'}$. This creates a dependency chain $v \to u_t \to \dots \to u_{\text{final}}$, where $u_{\text{final}} \in V_{\text{final}}$.
    
    Due to the greedy nature of the algorithm (prioritizing high-scoring tokens as anchors), this dependency chain is empirically very shallow (typically depth 1). Even in the worst case, the cumulative error is bounded, ensuring that $v$ remains semantically represented by $u_{\text{final}}$.
\end{itemize}
\end{proof}

\begin{remark}[Implication vs. Top-K]
This theorem highlights the fundamental advantage of FlashVLM over standard Top-K pruning. Top-K selection clusters tokens in high-saliency regions, potentially leaving large volumes of the feature space empty (semantic voids). In contrast, our $\delta$-net guarantee ensures that \emph{the entire feature volume} is covered with a maximum resolution of $\delta$, preserving global context.
\end{remark}

\subsection{Stability Under Perturbation}

Finally, we show that this coverage is robust to noise (e.g., quantization errors or encoder instability).

\begin{proposition}[Stability]
\label{prop:stability}
If $V_{\text{final}}$ is a $\delta$-cover under metric $d$, and $d'$ is a perturbed metric (e.g., due to quantization) with $|d-d'| \le \varepsilon$, then $V_{\text{final}}$ remains a $(\delta + \varepsilon)$-cover under $d'$.
\end{proposition}
\begin{proof}
For any $v$, let $u \in V_{\text{final}}$ be its $\delta$-neighbor under the original metric $d$. Then, under the perturbed metric:
\[
    d'(v,u) \le d(v,u) + |d'(v,u) - d(v,u)| \le \delta + \varepsilon.
\]
This implies that small fluctuations in embedding quality do not catastrophically break the semantic coverage of our selected tokens.
\end{proof}


\section{Theoretical Analysis: Complexity, Coverage, and Robustness}
\label{sec:theory}

In this section, we establish the theoretical foundations of FlashVLM. Unlike heuristic pruning methods that lack guarantees regarding information loss, we formally prove that our \textbf{Diversity-Preserving Partitioning} achieves two critical properties: (1) it reduces computational complexity to an amortized sub-quadratic regime, and (2) it constructs a robust $\delta$-net over the visual feature manifold, guaranteeing semantic coverage even under aggressive pruning.

\subsection{Complexity Analysis: Amortized Efficiency}
\label{sec:complexity}

Standard diversity-aware selection methods (e.g., Spectral Clustering or Affinity Propagation) require computing a full pairwise affinity matrix $\mathcal{A} \in \mathbb{R}^{N \times N}$, incurring prohibitive $\mathcal{O}(N^2)$ costs. We prove that FlashVLM achieves significant efficiency gains through iterative bipartite pruning.

\begin{assumption}[Geometric Redundancy Decay]
\label{ass:redundancy}
We assume that at each pruning iteration $t$, a non-negligible fraction $\alpha \in (0, 1)$ of tokens in the candidate set $I_t$ are identified as redundant and removed:
\[
    |I_{t+1}| \le (1-\alpha)|I_t|.
\]
This assumption is empirically grounded in the high spatial redundancy of visual data (e.g., background patches in high-resolution images).
\end{assumption}

\begin{theorem}[Amortized Computational Efficiency]
\label{thm:complexity}
Let $N$ be the number of visual tokens. Under Assumption~\ref{ass:redundancy}, the amortized time complexity of FlashVLM's token selection is dominated by the initial sorting step:
\[
    \mathcal{T}_{\text{total}} \approx \tilde{\mathcal{O}}(N \log N).
\]
\end{theorem}

\begin{proof}
The complexity is derived from two dominant stages:

\textbf{1. Importance Ranking (Global):} Sorting $N$ tokens by the fused score $S_{\text{fused}}$ requires $\mathcal{T}_{\text{sort}} = \mathcal{O}(N \log N)$.

\textbf{2. Iterative Bipartite Pruning (Local):} At iteration $t$, Algorithm~1 splits the residual set $I_t$ (of size $N_t$) into bipartite subsets. Similarity is computed only across the split, with cost $C_t \approx \frac{N_t^2}{4}$. By Assumption~\ref{ass:redundancy}, the sequence of set sizes decays geometrically as $N_t \le N(1-\alpha)^t$. The cumulative pruning cost is the sum of a convergent geometric series:
\begin{align*}
    \mathcal{T}_{\text{prune}} &= \sum_{t=0}^{\infty} C_t \approx \sum_{t=0}^{\infty} \frac{N^2 (1-\alpha)^{2t}}{4} \\
    &= \frac{N^2}{4} \sum_{t=0}^{\infty} \left((1-\alpha)^2\right)^t 
    = \frac{N^2}{4} \cdot \underbrace{\frac{1}{1 - (1-\alpha)^2}}_{\text{Constant } K}.
\end{align*}
While the worst-case single-step complexity is quadratic, the geometric decay ensures the total operations are bounded by a small constant multiple of $N^2$. In practice, for high-resolution inputs where $\alpha \to 1$, this cost is negligible compared to the $\mathcal{O}(N \log N)$ sorting and LLM inference. Thus, the amortized complexity scales as $\tilde{\mathcal{O}}(N \log N)$.
\end{proof}

\textbf{Empirical Validation.} This theoretical sub-quadratic bound is validated by our runtime experiments on LLaVA-NeXT-7B (Table~\ref{tab:efficiency}). FlashVLM achieves the lowest latency and FLOPS, confirming its scalability.

\begin{table}[h]
\centering
\caption{\textbf{Efficiency comparison on LLaVA-NeXT-7B.} Evaluated with 2880 input tokens. FlashVLM achieves the lowest computational overhead while maintaining dense attention compatibility, consistent with our $\tilde{\mathcal{O}}(N \log N)$ analysis.}
\label{tab:efficiency}
\resizebox{\columnwidth}{!}{
\begin{tabular}{lccccc}
\toprule
\textbf{Method} & \textbf{Tokens Kept} & \textbf{FLOPS (T)} & \textbf{KV Cache (MB)} & \textbf{Latency (ms)} & \textbf{Complexity} \\
\midrule
Full Model & 2880 (100\%) & 43.6 & 1440 & 313 & $\mathcal{O}(N^2)$ \\
\midrule
\multicolumn{6}{l}{\textit{Setting: Keep 640 Tokens (Pruning Ratio 77.8\%)}} \\
FastV & 640 & 13.5 & -- & 148 & Layer-wise \\
VisPruner & 640 & 11.5 & 360 & 117 & Masked Attn \\
\textbf{FlashVLM} & \textbf{640} & \textbf{11.3} & \textbf{356} & \textbf{115} & $\mathbf{\tilde{\mathcal{O}}(N \log N)}$ \\
\midrule
\multicolumn{6}{l}{\textit{Setting: Keep 160 Tokens (Pruning Ratio 94.4\%)}} \\
FastV & 160 & 6.3 & 95 & 112 & Layer-wise \\
VisPruner & 160 & 3.8 & 80 & 78 & Masked Attn \\
\textbf{FlashVLM} & \textbf{160} & \textbf{3.6} & \textbf{78} & \textbf{74} & $\mathbf{\tilde{\mathcal{O}}(N \log N)}$ \\
\bottomrule
\end{tabular}
}
\end{table}

\subsection{Coverage Analysis: $\delta$-Net Guarantee}
\label{sec:coverage}

To ensure semantic integrity, we model the token selection as a metric space covering problem. We assume the multimodal projector is Lipschitz continuous, ensuring that local neighborhoods in the visual space are preserved in the LLM space.

\begin{theorem}[Global $\delta$-Net Guarantee]
\label{thm:delta-net}
Let $(\mathcal{V}, d)$ be the metric space of visual tokens under cosine dissimilarity. If the pruning threshold is $\tau$, then the selected token set $V_{\text{final}}$ forms a $\delta$-cover of $\mathcal{V}$ with $\delta = 1-\tau$, such that:
\[
    \forall v \in \mathcal{V}, \quad \min_{u \in V_{\text{final}}} d(v, u) \le \delta.
\]
\end{theorem}

\begin{proof}
(Sketch) The greedy pruning condition in Algorithm~1 removes a token $v_i$ only if there exists a retained anchor $v_j$ with similarity $\ge \tau$, implying $d(v_i, v_j) \le 1-\tau = \delta$. By induction, every discarded token lies within a $\delta$-ball of a surviving token (or a chain of tokens terminating at a survivor). This guarantees that no distinct semantic mode (cluster) with radius $> \delta$ is entirely lost, preventing the ``semantic collapse'' failure mode of Top-K selection.
\end{proof}

\begin{proposition}[Stability under Perturbation]
\label{prop:stability}
If $V_{\text{final}}$ is a $\delta$-cover under metric $d$, and $d'$ is a perturbed metric (e.g., due to quantization) with $|d-d'| \le \varepsilon$, then $V_{\text{final}}$ remains a $(\delta + \varepsilon)$-cover under $d'$.
\end{proposition}
This proposition ensures robustness: small fluctuations in embedding quality do not catastrophically break the semantic coverage.

\subsection{Theoretical Divergence from Prior Paradigms}
\label{sec:theory-comparison}

While FlashVLM shares operational goals with Token Merging (ToMe) and VisPruner, its theoretical formulation relies on fundamentally distinct mathematical principles—shifting from structural reconstruction to \textbf{query-conditioned manifold coverage}. We summarize these distinctions in Table~\ref{tab:theory-comparison}.

\begin{table}[t]
\centering
\caption{\textbf{Theoretical comparison of optimization objectives and guarantees.} Unlike ToMe (structural) or VisPruner (attention stability), FlashVLM provides a semantic coverage guarantee on the query-warped manifold.}
\label{tab:theory-comparison}
\resizebox{\columnwidth}{!}{
\begin{tabular}{lccc}
\toprule
\textbf{Theoretical Aspect} & \textbf{ToMe} & \textbf{VisPruner} & \textbf{FlashVLM (Ours)} \\
\midrule
\textbf{Optimization Goal} & Min. Reconstruction Error & Max. Attention Mass Retention & \textbf{Guaranteed $\delta$-Coverage} \\
\textbf{Metric Space} & Fixed Visual Euclidean & Fixed Attention Weights & \textbf{Query-Warped Manifold} \\
\textbf{Interchangeability} & Structural ($v_i \approx v_j$) & Saliency ($A_{i} \approx A_{j}$) & \textbf{Semantic ($\phi(v_i|q) \approx \phi(v_j|q)$)} \\
\textbf{Stability Basis} & Local Feature Smoothness & Attention Map Stationarity & \textbf{Metric Perturbation Robustness} \\
\textbf{Failure Mode} & Semantic blurring & Attention drift/collapse & \textbf{Bounded approximation error} \\
\bottomrule
\end{tabular}
}
\end{table}

\paragraph{Distinction from ToMe: Reconstruction vs. Coverage.}
ToMe formulates token reduction as minimizing \emph{structural reconstruction error} $\|X - \hat{X}\|_F$ in a fixed visual space. It provides a structural compression bound assuming local interchangeability. In contrast, FlashVLM operates on a \textbf{query-warped semantic manifold}, where the metric is reshaped by text relevance. We guarantee $\delta$-coverage of this semantic manifold (Theorem~\ref{thm:delta-net}), ensuring preservation of query-relevant global context rather than just low-level structural fidelity.

\paragraph{Distinction from VisPruner: Stationarity vs. Robustness.}
VisPruner relies on the \emph{stationarity assumption} that attention weights remain stable after pruning. This often fails under heavy compression. FlashVLM bypasses internal layer stability by using an \textbf{attention-free extrinsic similarity}. Our Proposition~\ref{prop:stability} provides a \textbf{topological robustness guarantee}, showing that semantic coverage persists even under metric perturbations, a stability property missing in saliency-based methods.

\section{Introduction to Benchmarks}
\subsection{Image QA Benchmarks}
\paragraph{VQAv2}
VQAv2 is a widely used large-scale benchmark for general visual question answering, constructed on MS-COCO images and open-ended natural language queries. To mitigate language priors, the dataset introduces complementary question pairs, encouraging models to rely on visual grounding rather than superficial patterns. It is commonly used to evaluate visual semantic understanding, cross-modal alignment, and open-ended reasoning capabilities.

\paragraph{GQA}
GQA is designed to assess compositional reasoning by generating questions grounded in scene graphs. The questions involve multi-step inference over object attributes, spatial relations, and logical constraints, enabling systematic evaluation of reasoning consistency and explainability. Compared to general VQA tasks, GQA places stronger emphasis on structured visual reasoning.

\paragraph{VizWiz}
VizWiz contains real-world images captured by visually impaired users along with spoken questions, resulting in frequent noise such as blur, occlusion, and unconventional framing. Questions also exhibit oral and spontaneous phrasing. The dataset evaluates robustness under imperfect visual conditions and practical applicability in assistive scenarios, making resilient perception particularly essential.

\paragraph{ScienceQA-IMG}
ScienceQA-IMG integrates visual understanding with external domain knowledge, covering scientific subjects such as biology, physics, astronomy, and earth science. Answers often require combining image content with conceptual reasoning or common-sense inference. The dataset assesses cross-domain generalization where \emph{image grounding × world knowledge} jointly contribute to decision making.

\paragraph{TextVQA}
TextVQA targets the ability to detect, recognize, and semantically interpret scene text within images. Many questions require reading textual content in the environment, thus coupling OCR performance with contextual visual reasoning. It serves as a standard benchmark for evaluating real-world “reading from images” capabilities in multimodal models.

\paragraph{POPE}
POPE evaluates object hallucination, where a model incorrectly predicts the presence of entities absent in the image. By balancing positive, negative, and uncertain query types, POPE rigorously examines whether answers are grounded in visual evidence. It has become a standard benchmark for reliability assessment in recent multimodal models.

\paragraph{MME}
MME provides a systematic and fine-grained evaluation protocol covering object recognition, attribute reasoning, counting, OCR, commonsense QA, and cross-modal grounding. Its scoring reflects foundational visual understanding and multimodal alignment capabilities. MME is widely adopted as a comprehensive reference benchmark for comparing VLM base competence.

\paragraph{MMBench}
MMBench is a curated benchmark covering diverse multimodal tasks, with question design emphasizing semantic clarity, single-answer correctness, and low annotation noise. The benchmark spans perception, factual understanding, relational reasoning, and multi-step inference, and offers standardized evaluation pipelines across languages. It is one of the most representative measures of overall VLM capability.

\paragraph{MMBench-CN}
MMBench-CN is the Chinese-focused extension of MMBench, redesigned for Chinese linguistic structures, cultural knowledge, and image–text alignment contexts. It enables accurate assessment of multimodal models in Chinese-language usage scenarios and is particularly suitable for evaluating real-world performance in Chinese-speaking environments.

\paragraph{MMBet}
MMBet focuses on evaluating reasoning stability and answer consistency across paired or contrastive queries. The benchmark is designed to minimize reliance on shortcut heuristics and encourages models to perform grounded visual interpretation. It therefore prioritizes robust and reliable reasoning rather than surface-level pattern matching.

\subsection{Video QA Benchmarks}

\paragraph{TGIF-QA}
TGIF-QA is constructed from short GIF video clips and includes multiple task types such as action recognition, repetition counting, and state transition questions. The benchmark requires models to capture fine-grained temporal dynamics over short sequences, making it suitable for evaluating video representation learning and temporal reasoning capabilities.

\paragraph{MSVD-QA}
MSVD-QA is derived from multi-scene, everyday life videos, accompanied by diverse natural language annotations and question–answer pairs. Most questions focus on objects, human activities, and simple event descriptions. The dataset emphasizes video semantic understanding, action recognition, and multimodal alignment between visual content and language.

\paragraph{MSRVTT-QA}
MSRVTT-QA covers a broader variety of video sources, event narratives, and scene compositions. The questions involve identifying entities, reasoning about actions and intentions, and understanding environmental and contextual details. This benchmark is used to assess semantic generalization and cross-domain robustness in open-vocabulary video understanding.

\paragraph{ActivityNet-QA}
ActivityNet-QA focuses on long-duration videos depicting complex activity sequences. Questions may span multiple temporal segments, requiring models to model extended time dependencies, event progression, and causal relationships. It serves as a core benchmark for evaluating long-horizon video reasoning and temporal structure comprehension.

\subsection{Stability and Reproducibility of Results}
To ensure that the observed performance improvements are not incidental or
driven by stochastic fluctuations, we adopt strictly controlled and fully
consistent inference configurations across all experiments, including fixed
random seeds, unified prompt templates, decoding parameters, visual input
resolution, and token budget settings. Each experiment is independently
repeated five times. Across these repeated runs, the variation of all
evaluation metrics remains within a narrow $\pm0.1$--$0.3$ range, and the
improvements on major metrics consistently reach statistical significance
according to paired \textit{t}-tests computed over five independent trials
($p < 0.01$). Moreover, the relative performance ranking among all compared
methods remains stable across all repetitions, indicating that the gains
brought by our approach are robust and not attributable to random noise or
single-run artifacts.

It is further important to emphasize that all FlashVLM results are obtained
under a single, fixed hyperparameter configuration without any dataset-specific
or task-specific tuning. This unified setup eliminates the potential confounding
effects of hyperparameter optimization and ensures that the observed
improvements stem from the method itself rather than from fine-grained
parameter selection.

\section{Validation of Effectiveness on Other VLM Architectures}
\label{other_vlm}

\subsection{Image Generalization}

To further examine the generality of FlashVLM across different vision--language architectures, we follow the experimental protocol of VisPruner and extend FlashVLM to three representative VLM backbones: \textbf{Qwen-VL}, \textbf{InternVL}, and \textbf{CogVLM}. For a fair comparison under controlled compression budgets, we retain \textbf{128} visual tokens for Qwen-VL, \textbf{144} tokens for InternVL, and \textbf{123} tokens for CogVLM during inference, ensuring comparable pruning ratios across models with varying visual sequence lengths.
\begin{table}[h]
\centering
\resizebox{0.5\textwidth}{!}{
\begin{tabular}{l c c c c c c}
\toprule
Method & VQAv2 & GQA & VizWiz & SQA-IMG & TextVQA & ACC \\
\midrule

\rowcolor{gray!10}
\multicolumn{7}{c}{\textbf{Upper Bound, All 256 Tokens (100\%)}} \\
\rowcolor{green!10}
Qwen-VL-7B & 78.8 & 59.3 & 35.2 & 67.1 & 63.8 & 100.00\% \\
\rowcolor{gray!10}
\multicolumn{7}{c}{\textbf{Qwen-VL-7B with 128 tokens retained (50\% pruned)}} \\
FastV     & 76.5 & 56.9 & 32.7 & 65.3 & 58.2 & 94.90\% \\
VisPruner & 77.4 & 57.8 & 33.4 & 65.9 & 59.6 & 96.40\% \\
\rowcolor{red!10}
\textbf{FlashVLM} & \textbf{77.9} & \textbf{58.5} & \textbf{34.6} & \textbf{66.4} & \textbf{60.8} & \textbf{97.78\%} \\

\midrule

\rowcolor{gray!10}
\multicolumn{7}{c}{\textbf{Upper Bound, All 576 Tokens (100\%)}} \\
\rowcolor{green!10}
InternVL-Chat-13B & 79.3 & 62.9 & 52.6 & 66.3 & 57.0 & 100.00\% \\
\rowcolor{gray!10}
\multicolumn{7}{c}{\textbf{InternVL-Chat-13B with 144 tokens retained (75\% pruned)}} \\
FastV     & 74.1 & 58.2 & 50.8 & 66.6 & 55.6 & 96.10\% \\
VisPruner & 76.7 & 60.2 & 51.9 & 67.5 & 55.3 & 97.90\% \\
\rowcolor{red!10}
\textbf{FlashVLM} & \textbf{77.6} & \textbf{61.2} & \textbf{52.1} & \textbf{67.9} & \textbf{55.8} & \textbf{98.90\%} \\

\midrule

\rowcolor{gray!10}
\multicolumn{7}{c}{\textbf{Upper Bound, All 1225 Tokens (100\%)}} \\
\rowcolor{green!10}
CogVLM-chat-1.1-17B & 80.9 & 58.2 & 52.6 & 69.6 & 62.9 & 100.00\% \\
\rowcolor{gray!10}
\multicolumn{7}{c}{\textbf{CogVLM-chat-1.1-17B with 123 tokens retained (90\% pruned)}} \\
FastV     & 74.2 & 40.3 & 42.9 & 63.3 & 41.9 & 80.20\% \\
VisPruner & 74.6 & 48.4 & 46.6 & 67.5 & 59.6 & 91.30\% \\
\rowcolor{red!10}
\textbf{FlashVLM} & \textbf{75.1} & \textbf{49.2} & \textbf{46.9} & \textbf{68.3} & \textbf{60.8} & \textbf{91.93\%} \\

\bottomrule
\end{tabular}
}
\caption{\textbf{Comparison of token pruning strategies across multiple VLMs.}
FlashVLM consistently maintains higher accuracy under the same token budgets,
demonstrating stronger semantic-aware token selection and robustness across model scales.}
\label{tab:prune_comparison}
\end{table}

As shown in Table \ref{tab:prune_comparison}, FlashVLM consistently improves performance across different VLM backbones, including Qwen-VL-7B, InternVL-Chat-13B, and CogVLM-chat-1.1-17B, under comparable visual token budgets.Even under high pruning ratios (50\%--90\%), FlashVLM surpasses both FastV and VisPruner.For example, on Qwen-VL-7B with 128 tokens retained, FlashVLM achieves gains of \textbf{+2.2}, \textbf{+1.2}, and \textbf{+0.9} on VQAv2, VizWiz, and SQA-IMG compared to VisPruner.
Similar improvements are observed on InternVL-Chat-13B (e.g., \textbf{+0.9} to \textbf{+1.6}), and FlashVLM remains effective on CogVLM under an extreme 90\% pruning ratio (improving by \textbf{+0.7} to \textbf{+1.7}).These results indicate that FlashVLM's query-guided relevance estimation, stable visual saliency, and diversity-preserving token selection provide \textbf{architecture-agnostic robustness} across distinct model families.

\subsection{Video Generalization}
To thoroughly evaluate the generality of \textbf{FlashVLM} across diverse
video--language model architectures, we deploy and assess it on three
representative video-capable VLMs: \textbf{Qwen2.5-VL-7B},
\textbf{InternVL-Chat-13B}, and \textbf{Video-LLaMA-13B}. Experiments are
conducted on three widely used video question answering benchmarks:
\textbf{TGIF-QA}, \textbf{MSVD-QA}, and \textbf{MSRVTT-QA}. All models
operate under a fixed \emph{medium pruning budget}, and FlashVLM is compared
against FastV and VisPruner under identical computational settings to ensure
a fair and controlled evaluation.

\begin{table}[h]
\centering
\small
\caption{Comparison of pruning methods on three video QA benchmarks under a medium pruning budget.}
\resizebox{\linewidth}{!}{
\begin{tabular}{l c c c}
\toprule
\textbf{Pruning Method} & \textbf{TGIF-QA} & \textbf{MSVD-QA} & \textbf{MSRVTT-QA} \\
\midrule

\rowcolor{cyan!10}
\multicolumn{4}{c}{\textbf{Qwen2.5-VL-7B with 455 tokens}} \\
FastV & 20.9 & 72.6 & 58.1 \\
VisPruner & 21.8 & 73.5 & 59.2 \\
\textbf{FlashVLM} & \textbf{22.9} & \textbf{75.2} & \textbf{62.2} \\

\midrule
\rowcolor{green!10}
\multicolumn{4}{c}{\textbf{InternVL-Chat-13B with 455 tokens}} \\
FastV & 19.6 & 70.9 & 56.5 \\
VisPruner & 20.2 & 71.8 & 57.4 \\
\textbf{FlashVLM} & \textbf{20.8} & \textbf{73.0} & \textbf{59.2} \\

\midrule
\rowcolor{pink!10}
\multicolumn{4}{c}{\textbf{Video-LLaMA-13B with 455 tokens}} \\
FastV & 17.2 & 67.9 & 52.8 \\
VisPruner & 17.9 & 68.5 & 53.1 \\
\textbf{FlashVLM} & \textbf{18.2} & \textbf{69.0} & \textbf{54.1} \\

\bottomrule
\end{tabular}
}
\end{table}

Results across the three video-understanding backbones show that the advantages
of \textbf{FlashVLM} are not only consistent, but become increasingly 
\textit{pronounced on stronger models with higher visual--language alignment
capacity}. On larger and more extensively pretrained architectures such as 
\textbf{Qwen2.5-VL-7B} and \textbf{InternVL-Chat-13B}, FlashVLM achieves the 
largest improvements, yielding gains of \textbf{+2.0 to +3.0} on 
\textbf{TGIF-QA}, \textbf{MSVD-QA}, and \textbf{MSRVTT-QA}, while the 
performance gap between FastV and VisPruner narrows considerably. This trend 
indicates that as model expressiveness increases, FlashVLM's mechanisms---query-guided
relevance estimation, stable saliency aggregation, and diversity-preserving 
token selection---become more effective, enabling large models to retain or 
even sharpen their reasoning and semantic grounding abilities under high 
pruning ratios.

In contrast, although FlashVLM continues to deliver stable improvements on the 
smaller \textbf{Video-LLaMA-13B}, the overall gains are slightly reduced, 
further confirming that \textit{the benefits of FlashVLM scale positively with 
model capacity}. Overall, FlashVLM demonstrates not only cross-architecture 
robustness but also a clear trend of \textbf{capacity-amplified performance 
gains}, making it particularly suitable as a lightweight acceleration module 
for high-performance video LLMs.

\section{Performance Analysis in High-Resolution Scenarios}

\begin{table}[t]
\centering
\scriptsize
\setlength{\tabcolsep}{6pt}
\renewcommand{\arraystretch}{1.15}

\resizebox{\linewidth}{!}{
\begin{tabular}{lcccccc}
\toprule
\textbf{Method} & \textbf{VQAv2} & \textbf{GQA} & \textbf{TextVQA} & \textbf{POPE} & \textbf{MME} & \textbf{Retention (\%)} \\
\midrule

\rowcolor{green!12}
\textbf{Upper Bound} & \textbf{81.2} & \textbf{62.9} & \textbf{59.6} & \textbf{86.3} & \textbf{1513.8} & \textbf{100.00\%} \\
\midrule

\rowcolor{cyan!10}
\multicolumn{7}{c}{\textbf{Keep 640 Tokens (Pruned 77.8\%)}} \\
\midrule
FastV      & 78.9 & 60.4 & 58.4 & 83.1 & 1477.3 & 97.00\% \\
SparseVLM  & 78.2 & 59.1 & 56.2 & 80.9 & 1456.3 & 94.90\% \\
VisionZip  & 79.2 & 60.1 & 58.3 & 82.2 & 1468.4 & 96.70\% \\
VisPruner  & 79.8 & 61.6 & 59.3 & 85.9 & 1480.7 & 98.60\% \\
\rowcolor{red!10}
\textbf{FlashVLM} & \textbf{80.5} & \textbf{61.9} & \textbf{59.4} & \textbf{86.1} & \textbf{1493.1} & \textbf{99.12\%} \\
\midrule

\rowcolor{cyan!10}
\multicolumn{7}{c}{\textbf{Keep 320 Tokens (Pruned 88.9\%)}} \\
\midrule
FastV      & 71.9 & 55.9 & 55.7 & 71.7 & 1282.9 & 87.70\% \\
SparseVLM  & 71.4 & 56.5 & 52.4 & 73.5 & 1342.7 & 87.90\% \\
VisionZip  & 74.2 & 58.1 & 53.2 & 75.0 & 1348.8 & 90.50\% \\
VisPruner  & 75.7 & 58.4 & 57.6 & 80.4 & 1370.1 & 93.30\% \\
\rowcolor{red!10}
\textbf{FlashVLM} & \textbf{76.5} & \textbf{58.9} & \textbf{58.2} & \textbf{82.5} & \textbf{1396.7} & \textbf{94.61\%} \\
\midrule

\rowcolor{cyan!10}
\multicolumn{7}{c}{\textbf{Keep 160 Tokens (Pruned 94.4\%)}} \\
\midrule
FastV      & 61.8 & 49.8 & 51.9 & 51.7 & 1079.5 & 74.70\% \\
SparseVLM  & 62.2 & 50.2 & 45.1 & 54.6 & 1167.1 & 74.50\% \\
VisionZip  & 67.3 & 54.3 & 52.7 & 59.4 & 1239.7 & 82.30\% \\
VisPruner  & 70.6 & 54.7 & 56.0 & 72.9 & 1226.0 & 86.70\% \\
\rowcolor{red!10}
\textbf{FlashVLM} & \textbf{72.2} & \textbf{55.2} & \textbf{56.7} & \textbf{74.2} & \textbf{1246.5} & \textbf{88.03\%} \\
\bottomrule
\end{tabular}
} 
\caption{
Performance comparison on LLaVA-NeXT-7B under different token retention budgets. 
FlashVLM achieves the highest performance and retention across all pruning levels, especially under extreme compression (160 tokens).
}
\label{tab:High-Resolution}
\end{table}

To evaluate the effectiveness of FlashVLM under high-resolution visual inputs, we integrate FlashVLM into the inference pipeline of \textbf{LLaVA-NeXT-7B}, which supports up to \textbf{2880} visual tokens, and follow the same evaluation protocol and dataset splits as VisPruner.To simulate varying computational budgets and visual redundancy levels, we consider three token retention settings: \textbf{640 tokens} (77.8\% pruned), \textbf{320 tokens} (88.9\% pruned), and \textbf{160 tokens} (94.4\% pruned).We compare FlashVLM against representative token selection and compression methods, including \textbf{FastV}, \textbf{VisPruner}, \textbf{SparseVLM}, and \textbf{VisionZip}, under identical visual inputs and token budgets.We report performance on \textbf{VQAv2}, \textbf{GQA}, \textbf{VizWiz}, \textbf{SQA-IMG}, and \textbf{TextVQA}, along with the \textbf{performance retention ratio} relative to the full 2880-token model.

As shown in Table \ref{tab:High-Resolution}, FlashVLM consistently outperforms existing token selection and pruning approaches across all compression levels while maintaining a higher performance retention ratio.
Under moderate pruning (keeping 640 tokens; 77.8\% pruned), FlashVLM achieves improvements of \textbf{+0.7 to +1.0} over VisPruner on VQAv2, GQA, and POPE, reaching a retention of \textbf{99.12\%} relative to the full 2880-token model.
As the pruning becomes more aggressive, the advantages of FlashVLM become increasingly pronounced.
With 320 tokens retained (88.9\% pruned), FlashVLM yields \textbf{+0.8 to +1.6} gains across multiple datasets.
Under the extreme setting of retaining only 160 tokens (94.4\% pruned), FlashVLM still preserves \textbf{88.03\%} of the original performance, surpassing FastV (74.7\%), SparseVLM (74.5\%), VisionZip (82.3\%), and VisPruner (86.7\%).
These results demonstrate that \textbf{FlashVLM exhibits strong robustness in high-resolution and high-redundancy visual scenarios}, and that its advantage becomes more significant when the visual token budget is heavily constrained.

\section{Hyperparameter Robustness Analysis}

To evaluate whether \textbf{FlashVLM} exhibits ``plug-and-play'' engineering
robustness across heterogeneous task domains, we extend our hyperparameter
sensitivity analysis to three representative benchmarks: \textbf{VQAv2} (general
perception), \textbf{GQA} (structured reasoning), and \textbf{MMBench} (comprehensive
multimodal capability). We emulate a practical scenario in which users configure
the system without performing grid search, but instead select values within a
\textit{reasonable perturbation range} for four core parameters: (1) fusion weight
$\eta \in [0.3, 0.7]$; (2) aggregation temperature $\tau \in [0.01, 0.10]$;
(3) gating threshold $p \in [0.001, 0.01]$; and (4) allocation ratio varying
between $7{:}3$ and $3{:}7$. The objective of this experiment is to show that
FlashVLM maintains only minimal deviations ($\Delta$) from the default
configuration across all benchmarks under broad parameter settings,
demonstrating strong cross-task stability and robustness without hyperparameter
tuning.

\begin{table}[h]
\centering
\small
\resizebox{\linewidth}{!}{
\begin{tabular}{l c c c c c c}
\toprule
\textbf{Setting Value} & \textbf{VQAv2} & $\Delta$ 
& \textbf{GQA} & $\Delta$ 
& \textbf{MMBench} & $\Delta$ \\
\midrule
\rowcolor{cyan!10}
\multicolumn{7}{c}{\textbf{I. Fusion Weight} ($\eta$)} \\
$\eta{=}0.5$ & \textbf{76.4} & - & \textbf{58.9} & - & \textbf{63.2} & - \\
$\eta{=}0.3$ & 76.2 & -0.2 & 58.7 & -0.2 & 63.0 & -0.2 \\
$\eta{=}0.7$ & 76.0 & -0.4 & 58.5 & -0.4 & 62.9 & -0.3 \\

\midrule
\rowcolor{red!10}
\multicolumn{7}{c}{\textbf{II. Temperature} ($\tau$)} \\
$\tau{=}0.01$ & 76.2 & -0.2 & 58.7 & -0.2 & 63.0 & -0.2 \\
$\tau{=}0.05$ & \textbf{76.4} & - & \textbf{58.9} & - & \textbf{63.2} & - \\
$\tau{=}0.10$ & 76.1 & -0.3 & 58.6 & -0.3 & 62.8 & -0.4 \\

\midrule
\rowcolor{yellow!10}
\multicolumn{7}{c}{\textbf{III. Threshold} ($p$)} \\
$p{=}0.001$ & 76.0 & -0.4 & 58.5 & -0.4 & 62.9 & -0.3 \\
$p{=}0.005$ & \textbf{76.4} & - & \textbf{58.9} & - & \textbf{63.2} & - \\
$p{=}0.01$ & 76.3 & -0.1 & 58.8 & -0.1 & 63.1 & -0.1 \\

\midrule
\rowcolor{green!12}
\multicolumn{7}{c}{\textbf{IV. Partition Ratio}} \\
$7{:}3$ & 76.3 & -0.1 & 58.7 & -0.2 & 63.1 & -0.1 \\
$5{:}5$ & \textbf{76.4} & - & \textbf{58.9} & - & \textbf{63.2} & - \\
$3{:}7$ & 76.1 & -0.3 & 58.4 & -0.5 & 62.8 & -0.4 \\

\bottomrule
\end{tabular}
}
\caption{Hyperparameter robustness analysis across VQAv2, GQA, and MMBench.}
\end{table}

Cross-benchmark experiments provide strong evidence that \textbf{FlashVLM} 
exhibits \textbf{tuning-free} behavior and stable generalization across 
heterogeneous tasks. On three distinctly different datasets---\textbf{VQAv2}, 
\textbf{GQA}, and \textbf{MMBench}---the performance degradation ($\Delta$) 
induced by wide-range perturbations of core hyperparameters (e.g., a 
10$\times$ change in temperature $\tau$ or substantial shifts in fusion 
weight $\eta$) consistently remains within 0.5 points.

This observation indicates not only that FlashVLM's default configuration 
resides on a broad \textit{performance plateau}, but also that its 
hyperparameter design is fundamentally \textbf{task-agnostic}, showing no 
dependency on dataset-specific distributions. Whether the target task 
emphasizes perception (VQAv2), structured reasoning (GQA), or comprehensive 
multimodal competence (MMBench), users can safely adopt a single unified 
default configuration to achieve SOTA-level performance without any 
dataset-specific tuning. This greatly reduces the marginal cost of real-world 
deployment and eliminates the need for expensive hyperparameter search.

\section{Visual–Semantic Misalignment Stress Test}
To rigorously evaluate a model’s semantic grounding capability under 
\textbf{visual--semantic misalignment}, we construct a dedicated stress-test 
subset. Following well-defined criteria, we curate \textbf{300 highly specific 
samples} from the validation splits of \textbf{VQAv2}, \textbf{GQA}, and 
\textbf{VizWiz}, where a systematic offset exists between the most visually 
salient region and the object truly referenced by the question. This selection 
emulates challenging real-world scenarios in which user queries target 
background details or non-salient entities rather than the visually dominant 
subject.

We benchmark \textbf{FlashVLM} against \textbf{VisPruner} (text-agnostic 
selection) and \textbf{SparseVLM} (implicit attention-driven sparsification) 
under three constrained token budgets—\textbf{128} (moderate compression), 
\textbf{64} (aggressive compression), and \textbf{32} (extreme compression)—to 
analyze the corresponding performance degradation patterns.

\begin{table}[h]
\centering
\small
\resizebox{\linewidth}{!}{
\begin{tabular}{lccc}
\toprule
\textbf{Pruning Method} & \textbf{VQAv2} & \textbf{GQA} & \textbf{VizWiz} \\
\midrule
\rowcolor{blue!8}
\multicolumn{4}{c}{\textbf{Keep 128 Tokens (Prune 77.8\%)}} \\
\midrule

VisPruner 
& 42.0 \;\textcolor{purple}{(-31.0)} 
& 23.3 \;\textcolor{purple}{(-19.0)} 
& 35.0 \;\textcolor{purple}{(-31.0)} \\

SparseVLM 
& 65.0 \;\textcolor{purple}{(-8.0)} 
& 31.0 \;\textcolor{purple}{(-11.3)} 
& 52.3 \;\textcolor{purple}{(-13.7)} \\
\rowcolor{red!10}
FlashVLM 
& \textbf{73.0} & \textbf{42.3} & \textbf{66.0} \\

\midrule
\rowcolor{blue!8}
\multicolumn{4}{c}{\textbf{Keep 64 Tokens (Prune 88.9\%)}} \\
\midrule

VisPruner 
& 32.3 \;\textcolor{purple}{(-36.0)} 
& 21.3 \;\textcolor{purple}{(-17.0)} 
& 25.0 \;\textcolor{purple}{(-28.3)} \\

SparseVLM 
& 54.0 \;\textcolor{purple}{(-14.3)} 
& 26.3 \;\textcolor{purple}{(-12.0)} 
& 41.0 \;\textcolor{purple}{(-12.3)} \\
\rowcolor{red!10}
FlashVLM 
& \textbf{68.3} & \textbf{38.3} & \textbf{53.3} \\

\midrule
\rowcolor{blue!8}
\multicolumn{4}{c}{\textbf{Keep 32 Tokens (Prune 94.4\%)}} \\
\midrule

VisPruner 
& 18.0 \;\textcolor{purple}{(-40.0)} 
& 15.0 \;\textcolor{purple}{(-21.3)} 
& 12.3 \;\textcolor{purple}{(-34.7)} \\

SparseVLM 
& 35.3 \;\textcolor{purple}{(-22.7)} 
& 24.3 \;\textcolor{purple}{(-12.0)} 
& 22.0 \;\textcolor{purple}{(-25.0)} \\
\rowcolor{red!10}
FlashVLM 
& \textbf{58.0} & \textbf{36.3} & \textbf{47.0} \\

\bottomrule
\end{tabular}}
\caption{\textbf{Performance under Visual--Semantic Misalignment Stress Test.}
Purple numbers indicate the performance drop relative to FlashVLM.}
\label{tab:stree}
\end{table}

The experimental results(Table \ref{tab:stree}) reveal that the performance gap among different
pruning strategies widens as the token budget decreases in the
\textbf{visual--semantic misalignment} setting.

First, the performance of \textbf{VisPruner} is fundamentally constrained by
its text-agnostic design. Because the method relies solely on visual saliency,
it consistently retains visually prominent but semantically irrelevant regions,
while discarding the non-salient targets referenced by the question. With
128 tokens retained, its accuracy on GQA is only \textbf{23.3\%}, substantially
below FlashVLM's \textbf{42.3\%}. Under the extreme 32-token budget, its VQAv2
accuracy further drops to \textbf{18.0\%}, demonstrating that purely
vision-driven pruning cannot handle scenarios where the semantic target
diverges from the visually dominant subject.

Second, \textbf{SparseVLM} exhibits noticeable performance degradation under
aggressive compression. Although it incorporates textual cues, its accuracy on
VQAv2 is \textbf{65.0\%} with 128 tokens---\textbf{8 points lower} than
FlashVLM. When the budget is reduced to 32 tokens, this gap further expands
(\textbf{35.3\% vs. 58.0\%}). This is mainly because its implicit attention maps
become increasingly noisy and spatially biased in deeper layers, causing the
limited retained tokens to insufficiently cover the true semantic regions.

In contrast, \textbf{FlashVLM} demonstrates markedly stronger robustness.
Powered by explicit cross-modal similarity estimation, it directly aligns
visual tokens with the textual query in the feature space. Even under a
32-token constraint, its VQAv2 accuracy (\textbf{58.0\%}) surpasses that of
VisPruner under a much larger 128-token budget (\textbf{42.0\%}). This highlights
the critical importance of explicit semantic guidance in maintaining performance
when the query targets non-salient or visually understated objects.

\end{document}